\documentclass[11pt]{article}
\usepackage[sort&compress,square,comma,authoryear]{natbib}
\usepackage[a4paper, total={6.1in, 10in}]{geometry}

\usepackage[utf8]{inputenc}
\usepackage[T1]{fontenc}
\usepackage{hyperref}
\usepackage{url}
\usepackage{booktabs}
\usepackage{tabularx}
\newcolumntype{Y}{>{\centering\arraybackslash}X}
\newcolumntype{Z}{>{\raggedleft\arraybackslash}X}
\usepackage[parfill]{parskip}
\usepackage{amsfonts}
\usepackage{nicefrac}
\usepackage{microtype}
\usepackage[table,dvipsnames]{xcolor}
\usepackage{amsmath,amssymb,amsthm}
\usepackage{bbold}
\usepackage{bbm}
\usepackage{colortbl}
\usepackage{graphicx}
\graphicspath{{./}}
\usepackage{mathrsfs,mathtools}
\usepackage{subfigure}
\usepackage{wrapfig}
\usepackage{caption}
\usepackage[capitalize]{cleveref}
\usepackage{tikz}
\usetikzlibrary{calc, fadings, decorations, shapes, positioning, arrows}
\usepackage[graphicx]{realboxes}
\usepackage{enumitem}
\usepackage{algorithm}
\usepackage{algpseudocode}
\usepackage{makecell}

\definecolor{salmon}{RGB}{234,153,153}
\definecolor{cornflowerblue}{RGB}{100,149,237}
\hypersetup{
    colorlinks,
    linkcolor={cornflowerblue},
    citecolor={cornflowerblue},
    urlcolor={salmon}
}

\theoremstyle{plain}
\newtheorem{theorem}{Theorem}
\newtheorem{proposition}[theorem]{Proposition}

\theoremstyle{definition}

\theoremstyle{remark}

\definecolor{RefShade}{HTML}{F7F7F7}
\definecolor{BASShade}{HTML}{EEF5FF}
\definecolor{VanillaShade}{HTML}{F7F7F7}

\newcommand{\best}[1]{\textbf{#1}}
\newcommand{\second}[1]{\underline{#1}}

\definecolor{DeltaPos}{HTML}{1B7F3B}
\definecolor{DeltaNeg}{HTML}{B23A2A}
\definecolor{DeltaNear}{HTML}{6E6E6E}
\definecolor{DeltaTime}{HTML}{1B7F3B}

\newcommand{\dpos}[1]{\textcolor{DeltaPos}{\textbf{#1}}}

\newcommand{\dnear}[1]{\textcolor{DeltaNear}{#1}}
\newcommand{\dtime}[1]{\textcolor{DeltaTime}{\textbf{#1}}}

\title{BASIS: Batchwise Advantage Estimation from Single-Rollout Information Sharing for LLM Reasoning}
\date{}

\usepackage{authblk}

\author[1]{Shijin Gong\textsuperscript{*}}
\author[2]{Erhan Xu\textsuperscript{*}}
\author[2]{Kai Ye\textsuperscript{*}}
\author[2]{Giulia Livieri\textsuperscript{\ensuremath{\dagger}}}
\author[3]{Francesco Quinzan\textsuperscript{\ensuremath{\dagger}}}
\author[2]{Chengchun Shi\textsuperscript{\ensuremath{\dagger}}}

\affil[1]{School of Management, University of Science and Technology of China}
\affil[2]{Department of Statistics, London School of Economics and Political Science}
\affil[3]{Department of Engineering Science, University of Oxford}
\affil[ ]{\textsuperscript{*}These authors contributed equally and are listed in alphabetical order.
\quad \textsuperscript{\ensuremath{\dagger}}Joint senior authors and corresponding authors.
Emails: g.livieri@lse.ac.uk, francesco.quinzan@eng.ox.ac.uk, c.shi7@lse.ac.uk}

\begin{document}

\maketitle

\begin{abstract}
Reinforcement learning with verifiable rewards has become a standard recipe for improving the reasoning abilities of large language models. Existing algorithms face a tradeoff between computational efficiency and sample efficiency in value estimation and policy learning. We introduce BASIS, a critic-free post-training algorithm designed to address this tradeoff. At each online training step, BASIS samples only one rollout per prompt, but leverages rich information across prompts in the entire batch to improve value function estimation. Our experiments demonstrate that BASIS reduces MSE in value function estimation by {\bf 69}\% compared to REINFORCE++, a representative single-rollout baseline, and achieves lower MSE with one rollout than GRPO estimators with {\bf 8} rollouts. This improvement in value estimation translates to better policy optimization: using substantially less online training time, BASIS achieves performance close to multi-rollout GRPO-type baselines and often outperforms single-rollout REINFORCE-type baselines.
\end{abstract}

\section{Introduction}
Recent progress in large language model (LLM) reasoning has increasingly relied on large-scale reinforcement learning (RL). In particular, RL with verifiable rewards \citep[RLVR,][]{lambert2024tulu3} scores LLM responses using automated task verifiers and uses these scores as rewards to update the model’s policy. Popularized by DeepSeekMath and DeepSeek-R1 \citep{shao2024deepseekmath,deepseek2025r1}, RLVR has become a standard post-training framework for open-source large reasoning models \citep[see e.g.,][]{liu2025fin,yu2025dapo,zheng2025gspo}, with numerous follow-up works appearing shortly thereafter (Section \ref{sec:relatedwork}). 

Existing RLVR algorithms face a fundamental tradeoff between computational efficiency and sample efficiency. More accurate value or advantage estimates can improve the sample efficiency of the resulting policy optimization algorithm, but they are often expensive to obtain. PPO-type algorithms train an auxiliary critic network for value estimation \citep{ouyang2022instructgpt}, while GRPO-type algorithms sample multiple rollouts per prompt and use their average reward for the estimation \citep{shao2024deepseekmath,liu2025drgrpo,chu2025gpg,zheng2025gspo}. These algorithms improve sample efficiency at the cost of additional computation. In contrast, single-rollout critic-free algorithms are computationally cheaper, but their value estimates are less accurate, leading to noisier policy updates.

In response to this tradeoff, we propose BASIS, short for \textbf{B}atchwise \textbf{A}dvantage estimation from \textbf{S}ingle-rollout \textbf{I}nformation \textbf{S}haring, a critic-free RLVR algorithm that samples only one rollout per prompt at each online training step. 
\begin{itemize}[leftmargin=*]
    \item \textbf{Methodologically}, BASIS introduces a novel batchwise advantage estimator that borrows rich information across prompts in the whole batch to improve value and advantage estimation (see Figure~\ref{fig:bas_pipeline} for an illustration). Together with offline  estimation and online calibration (Section \ref{sec:method-section}), this enables BASIS to retain much of the sample efficiency of multi-rollout algorithms while using only a single rollout during online training.

    \item \textbf{For value estimation}, BASIS is highly sample efficient: it substantially reduces the MSE of single-rollout baselines and achieves lower MSE than multi-rollout baselines using {\bf 8} rollouts. Moreover, unlike existing baselines, BASIS is robust to reward heterogeneity and prompt difficulty, while  producing informative advantage estimates (Section \ref{sec:offline-diagnostics}).

    \item \textbf{For policy optimization}, the BASIS advantage estimator can serve as a plug-in component for a range of RLVR algorithms. It stabilizes training and mitigates collapse during optimization. The resulting fine-tuned model often outperforms single-rollout baselines and remains competitive with multi-rollout baselines while requiring considerably less training time (Section~\ref{sec:policyopt}).
\end{itemize}

\begin{figure*}[t]
    \centering
    \includegraphics[width=\textwidth]{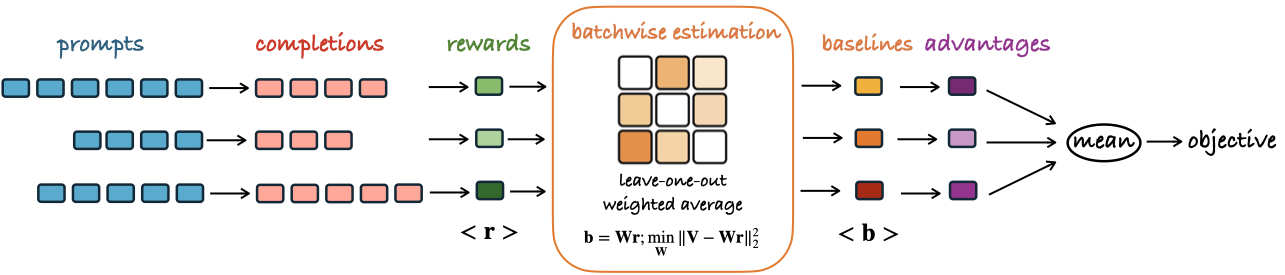}
    \caption{Overview of BASIS for constructing advantage estimates. BASIS samples one completion per prompt and uses reward information from the whole batch to estimate value baselines. The baselines are weighted averages of batch rewards, with the weight matrix $W$ shown in the center. The diagonal entries of $W$ are zero, shown in white, enforcing a prompt’s own reward to be excluded when estimating its baseline. Unlike single-rollout baselines such as REINFORCE++, which use a simple global average, BASIS chooses $W$ by minimizing the MSE of the value baselines. Advantages are then computed by subtracting the baselines from the observed rewards.}
    \label{fig:bas_pipeline}
\end{figure*}
\section{Related Work}\label{sec:relatedwork}
The algorithmic foundations of RLVR trace back to a long line of policy gradient algorithms in the RL literature \citep{sutton1999policy,williams1992reinforce}. One central theme in this literature is variance reduction: more accurate value or advantage estimates can reduce the variance of policy gradient estimation and improve the quality of policy optimization \citep{greensmith2004variance}. Motivated by this observation, existing approaches can be divided into four categories:
\begin{enumerate}[leftmargin=*]
    \item The first category, represented by PPO and SAO \citep{hou2026single}, employs a deep neural network to learn the value function and applies generalized advantage estimation for variance reduction \citep{schulman2016gae,schulman2017ppo}. While effective, this approach introduces an additional value network, making training substantially more costly.
    \item The second category, represented by GRPO, removes the value network and uses the average reward over multiple rollouts of the same prompt as a proxy for the value function \citep{shao2024deepseekmath,ahmadian2024rloo}.
    \item The third category follows the critic-free spirit of the second category, but aims to reduce the number of rollouts required per prompt. One approach uses the reward of a greedy, deterministic rollout, which only requires one rollout per prompt \citep{li2023remax}. Another approach reduces the required number of rollouts while maintaining sample efficiency by borrowing reward information either across training iterations for the same prompt \citep{wang2025kalman,xu2025single,gong2026kernelized,ruan2026spo++}, or across different prompts in the same batch \citep{hu2025reinforce,zeng2025shrinking,han2026ebpo}. 
    \item The last category studies variance reduction baselines beyond standard value functions \citep{hao2025on}. 
\end{enumerate} 
BASIS is most closely related to the third category, especially those methods that exploit batch-level reward information. However, as our experiments show, BASIS achieves substantially smaller MSE in value estimation than REINFORCE++ \citep{hu2025reinforce}, a representative batch-level baseline.

Beyond value and advantage estimation, existing work also studies exploration, clipping, rejection sampling, experience reuse, rollout down-sampling and pruning, uncertainty and difficulty-aware updates \citep{chen2025seed,dai2025stable,lin2025cppo,shrivastava2025sample,su2025klear,xiong2025a,xu2025not,zhang2025grpo,zhan2025exgrpo,cheng2026reasoning}. Since BASIS only modifies the value function baseline, it can serve as a plug-in component for many of these RLVR pipelines. 

Finally, recent work has begun to study the theoretical properties of RLVR   \citep{pang2025on,davis2025what,vojnovic2025what,yang2026your,huang2026learning,zhou2026demystifying}.


\section{Preliminaries}
We first introduce notation that will be used throughout the paper. We then describe how the classical REINFORCE algorithm \citep{williams1992reinforce} applies to LLM post-training. Finally, we use PPO, GRPO, and REINFORCE++ as representative variance reduction baselines to illustrate the first three categories of algorithms reviewed in Section \ref{sec:relatedwork}.  

\smallskip

\noindent \textbf{Notation}. We use $x\sim \mathcal{D}$ to denote a prompt drawn from a training dataset $\mathcal{D}$, and $y \sim \pi_\theta(\bullet|x)$ to denote a rollout sampled from the LLM policy
$\pi_\theta$ that we wish to fine-tune, with parameters $\theta$. 
For reasoning tasks, $y$ contains both a reasoning trace and a final answer.
An automated task verifier returns a scalar reward $r(x,y)\in\mathbb{R}$, where the reward function $r$ usually measures the correctness of the final answer. Let $\mathcal{B}=\{x_i\}_{i=1}^B$ denote a batch of prompts. 
To simplify the notation, we use $r_i$ to denote $r(x_i,y_i)$ for a given prompt-rollout pair $(x_i,y_i)$. We use $\theta_t$ to denote the model parameters at training iteration $t$, and define $V_{i,t}=\mathbb{E}_{y\sim \pi_{\theta_t}(\bullet|x_i)}[r(x_i,y)]$, $A_{i,t}=r_i-V_{i,t}$ as the policy's value and advantage, respectively.

\smallskip

\noindent \textbf{REINFORCE}. RLVR algorithms fine-tune the parameters $\theta$ by maximizing the expected reward
\begin{equation}\label{eqn:Jtheta}
    \mathcal J(\theta)
=
\mathbb E_{x\sim\mathcal D}
\Big\{\mathbb E_{y\sim\pi_\theta(\bullet|x)}
\bigl[r(x,y)\bigr]\Big\},
\end{equation}
which can be optimized using stochastic gradient methods \citep{robbins1951stochastic}. Specifically, its gradient $\nabla_\theta \mathcal J(\theta)$ can be shown to equal
$$
\mathbb E_{x\sim\mathcal D}
\Big\{\mathbb E_{y\sim\pi_\theta(\bullet|x)}
\bigl[r(x,y)\nabla_{\theta}\log \pi_{\theta}(y|x) \bigr]\Big\}.
$$
This identity motivates REINFORCE-type algorithms: starting from an initial parameter $\theta_1$, at each training step $t$, the algorithm samples a batch of prompts $\mathcal{B}=\{x_i\}_{i=1}^B$ and one rollout $y_i\sim \pi_{\theta_t}(\bullet|x_i)$ per prompt. It then estimates $\nabla_{\theta} \mathcal J(\theta_t)$ by averaging the product of the reward $r_i$ and policy score $\nabla_{\theta} \log \pi_{\theta_t}(y_i|x_i)$ over the batch, and applies stochastic gradient ascent to update $\theta_t$ to $\theta_{t+1}$.

\smallskip

\noindent \textbf{Variance reduction baselines}. The REINFORCE policy gradient estimator is unbiased, but often suffers from high variance, especially when the model is uncertain and sampled rollouts can vary substantially in quality. This variability leads to noisy rewards and gradient estimates, making policy optimization less stable. A standard remedy is to construct an advantage estimate $r_i-b_i$ by subtracting a baseline $b_i$ from the reward $r_i$ before multiplying by the policy score. 
Under suitable conditions, the optimal baseline that minimizes the variance of the policy gradient estimator is the value function \citep{greensmith2004variance}. 

PPO, GRPO, and REINFORCE++ all adopt this variance reduction principle, but differ in how they construct $b_i$, and hence the resulting advantage estimate. PPO sets the baseline to an estimated value function learned by an auxiliary neural network, which improves sample efficiency but introduces substantial computational cost. GRPO sets $b_i$ to the average reward over multiple rollouts from the same prompt, which eliminates the need for a value network but also increases computation through repeated rollout sampling. REINFORCE++ instead uses a global batch baseline, given by a simple average of $\{r_i\}_i$, which avoids both a value network and multiple rollouts per prompt, but the resulting baseline is shared across all prompts and may poorly approximate each prompt's individual value.

These choices highlight the central tradeoff addressed in this paper: accurate baseline estimation can be obtained by learning a value model or by repeating rollouts for each prompt, but both approaches increase computation.

\section{BASIS}
\label{sec:method-section}
We now introduce BASIS to address the tradeoff between accurate baseline estimation and computational efficiency. At each online training step, BASIS samples only one rollout per prompt while leveraging the rich information available across the entire training batch to improve each prompt's value and advantage estimation. 

BASIS proceeds in three steps: offline value estimation, batchwise refinement, and online calibration. It first constructs an initial prompt-level value estimate through offline estimation. It then refines this estimate online using reward information from the entire training batch. Finally, it calibrates the refined estimator at each training step. Algorithm \ref{alg:bas} summarizes the procedure for each training step. 
Below, we first describe the second step, which is the key ingredient of BASIS. We then explain how the offline value estimates are computed, and how the refined estimators are calibrated.

\smallskip

\noindent \textbf{Batchwise refinement}. At each training step $t$, given a prompt-rollout batch $\{(x_i,y_i)\}_{i=1}^B$, BASIS estimates the value $V_{i,t}$ for each $x_i$. 
Its main idea is simple: each $V_{i,t}$ is estimated by a weighted average of rewards from  other prompts in the batch:
\begin{equation}\label{eq:linear-baseline}
\widetilde{V}_{i,t}=\sum_{j\ne i} w_{ij} r_j,
\end{equation}
where the weights $\{w_{ij}\}_{j\ne i}$ are prompt-dependent. Such a linear combination allows BASIS to borrow information across prompts in the same batch, yielding a low-MSE estimate from only one rollout per prompt, as we show in the experiments.

Before specifying the weights $w_{ij}$, we highlight two features of this estimator. (i) It is prompt-dependent: the weights vary with $i$, unlike the global batch baseline used by REINFORCE++. With a proper choice of weights, this allows BASIS to better approximate the prompt-level value. (ii) It follows the leave-one-out principle of RLOO \citep{ahmadian2024rloo}: the reward $r_i$ of the target prompt is excluded from its own baseline. 

It remains to specify the weights $w_{ij}$. The following proposition motivates our choice.
\begin{proposition}[Best linear unbiased estimator (BLUE)]\label{proBLUB}
    For each prompt $i$, among all linear estimators of the form \eqref{eq:linear-baseline} that are unbiased for $V_{i,t}$, i.e., $\mathbb{E}(\widetilde{V}_{i,t}|\mathcal{B})=V_{i,t}$, the weights minimizing $\mathbb{E} [(V_{i,t}-\widetilde{V}_{i,t})^2|\mathcal{B}]$ are
    \begin{equation}\label{eqn:wij}
        w_{ij}=\frac{V_{i,t}V_{j,t}/\sigma_j^2}
{\sum_{k\ne i} V_{k,t}^2/\sigma_{k}^2},
\quad \forall j\in \{1,\cdots,B\}\setminus\{i\},
    \end{equation}
where $\sigma_j^2=\operatorname{Var}(r_j|x_j)$.
\end{proposition}
Proposition \ref{proBLUB} characterizes the ideal weights one would use to minimize the MSE of the resulting value estimator while preserving its unbiasedness. Intuitively, when estimating the value of the $i$th prompt, the weight increases with both the target value $V_{i,t}$ and the source value $V_{j,t}$, and decreases with the source reward variance $\sigma_j^2$.

Of course, the BLUE weights cannot be used directly during online training, because the true values $V_{i,t}$ and variances $\sigma_i^2$ are unknown. BASIS therefore uses this proposition as a guiding principle. For binary rewards, given an initial value estimate $\widehat V_{i,t}$, we estimate the reward variance by 
$\widehat{\sigma}_i^2= \widehat{V}_{i,t}(1-\widehat V_{i,t})$. We then plug these estimates into \eqref{eqn:wij} to construct the refined estimator $\widetilde{V}_{i,t}$.

When some initial values are close to zero or one, their Bernoulli variance estimates are close to zero, making the weights in \eqref{eqn:wij} numerically unstable. To address this, we define the following active set
\begin{equation*}
\mathcal A_t=\{i:\epsilon<\widehat V_{i,t}<1-\epsilon\},
\end{equation*}
for some small threshold $\epsilon>0$. For prompts in the active set, the sums in \eqref{eq:linear-baseline} and the denominator of \eqref{eqn:wij} are taken over $\mathcal A_t\setminus\{i\}$ rather than all prompts. For prompts outside the active set, BASIS falls back to the conservative zero baseline $b_i=0$.

We will describe how the initial value estimates are obtained below. To conclude this batchwise refinement step, we note that these initial estimates could also be used directly for advantage estimation without refinement. However, the batchwise refinement yields a substantially more accurate estimator by leveraging information shared across prompts. Figure \ref{fig:online-weighting-cached-value} illustrates this effect by comparing the MSEs of the initial and refined estimators at three checkpoints. It can be seen that the refined estimator achieves considerably lower MSE and is much less sensitive to a calibration hyperparameter.

\begin{figure*}[t]
\centering
\begin{minipage}[t]{0.32\linewidth}
\centering
\includegraphics[width=\linewidth]{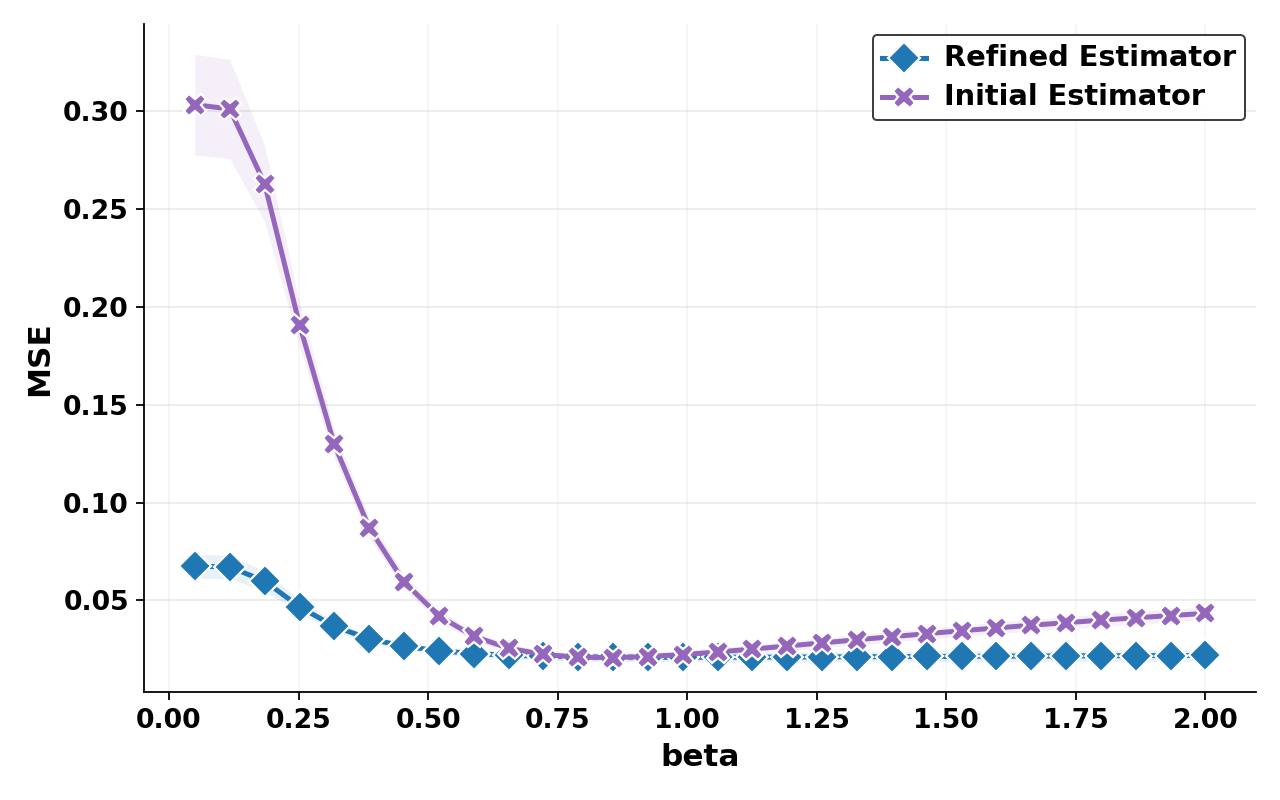}\\[-0.25em]
\end{minipage}
\hfill
\begin{minipage}[t]{0.32\linewidth}
\centering
\includegraphics[width=\linewidth]{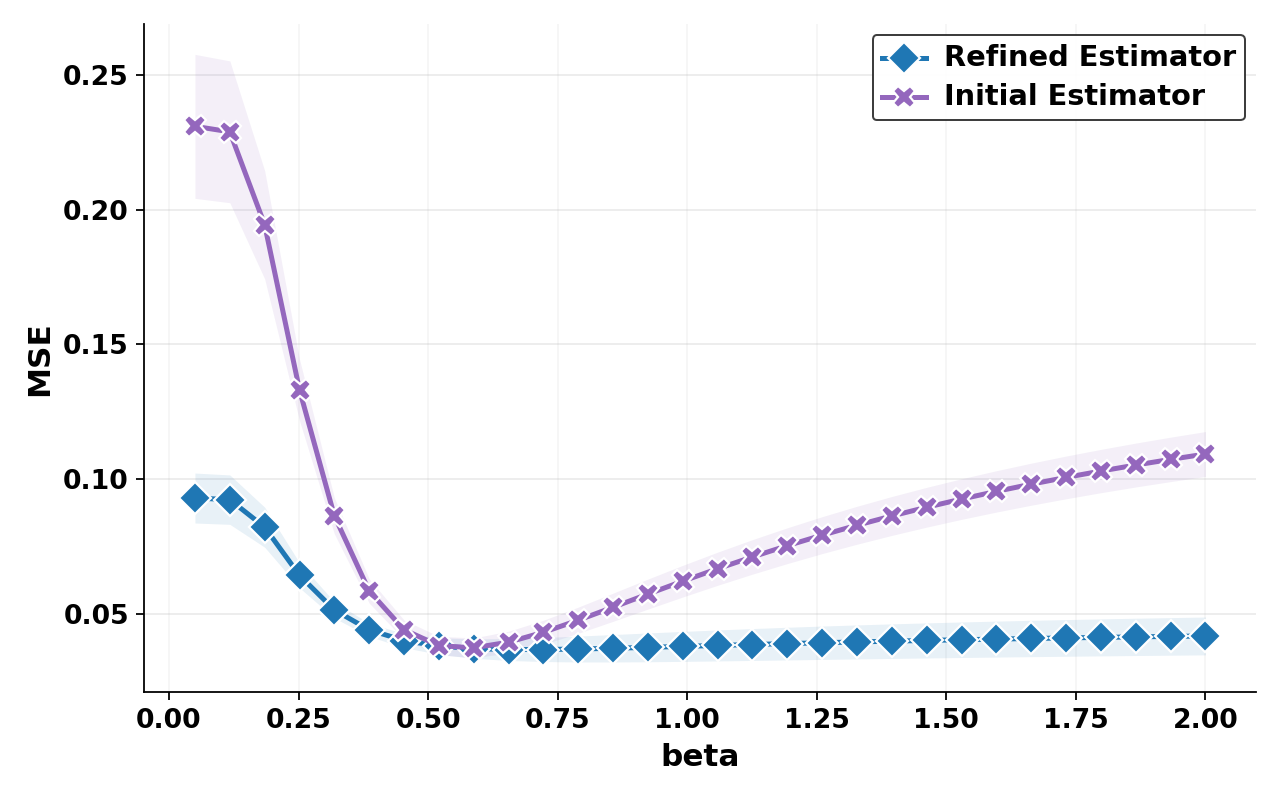}\\[-0.25em]
\end{minipage}
\hfill
\begin{minipage}[t]{0.32\linewidth}
\centering
\includegraphics[width=\linewidth]{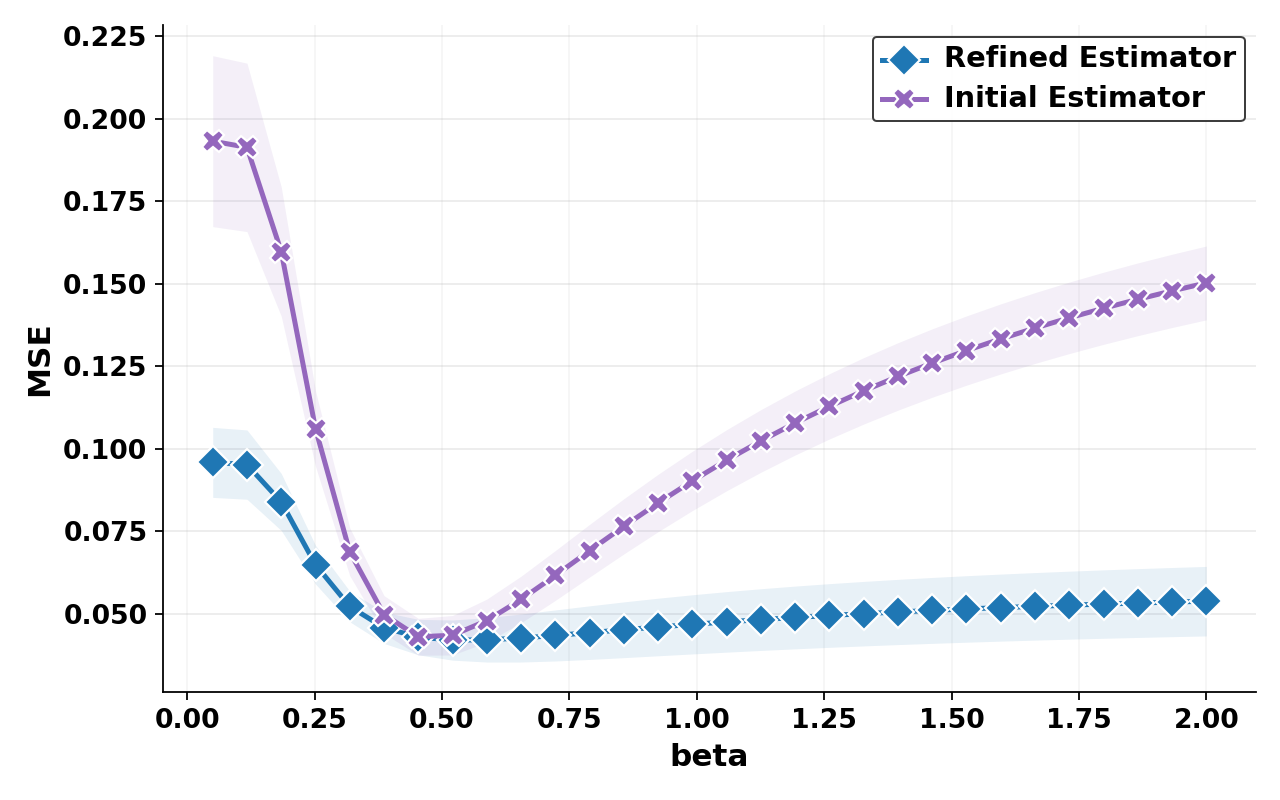}\\[-0.25em]
\end{minipage}

\caption{MSEs of the initial and refined estimators at three model checkpoints. The left panel reports MSE at an early training stage, while the middle and right panels report MSE at later stages. Both estimators depend on a hyperparameter $\beta$, and the curves visualize their MSE across different values of $\beta$.}
\label{fig:online-weighting-cached-value}
\end{figure*}

\smallskip 

\noindent \textbf{Offline value estimation}. We next describe how BASIS computes the initial value estimates $\widehat{V}_{i,t}$. The idea is motivated by direct preference optimization \citep{rafailov2023dpo} and its extension to LLM reasoning \citep{brantley2025accelerating}. To begin with, consider the following Kullback--Leibler (KL)-regularized objective function, 
\[
\mathbb{E}_{x\sim \mathcal{D}}
\mathbb{E}_{y\sim \pi(\bullet|x)}[r(x,y)
-
\beta\,\mathrm{KL}\!\left(\pi(\bullet|x)\,\|\,\pi_{\mathrm{ref}}(\bullet|x)\right)],
\]
where $\pi_{\mathrm{ref}}=\pi_{\theta_1}$ denotes the reference policy, taken to be the initial policy before fine-tuning, and $\beta>0$ controls the strength of the KL penalty. 

This objective is particularly appealing because its maximizer $\pi_\beta^*$ admits the following closed-form expression \citep{rafailov2023dpo}:
\begin{equation*}
\pi_\beta^*(y|x)
=
\frac{\pi_{\mathrm{ref}}(y|x)}{Z_\beta(x)}
\exp\Big(\frac{r(x,y)}{\beta}\Big),
\end{equation*}
where $Z_\beta(x)$ is a normalizing constant. This expression connects $\pi_{\beta}^*$ directly to $\pi_{\mathrm{ref}}$. As a result, estimating the value of $\pi_\beta^*$ does not require training or sampling from $\pi_\beta^*$ itself. Instead, it can be done offline by sampling rollouts from $\pi_{\mathrm{ref}}$  \citep{brantley2025accelerating}. We formalize this observation in the following proposition. 
\begin{proposition}[Closed-form value under $\pi_\beta^*$]\label{prop:soft-value}
For any prompt $x$ and $\beta>0$, the value function under $\pi_{\beta}^*$ equals
\begin{eqnarray}\label{eqn:Vbetaxstar}
\begin{split}
&V_\beta^*(x)
:=
\mathbb{E}_{y\sim \pi_\beta^*(\bullet|x)}[r(x,y)]\\
=&
\frac{
\mathbb E_{y\sim\pi_{\mathrm{ref}}(\bullet\mid x)}
\!\left[r(x,y)\exp(r(x,y)/\beta)\right]
}{
\mathbb E_{y\sim\pi_{\mathrm{ref}}(\bullet\mid x)}
\!\left[\exp(r(x,y)/\beta)\right]
}.
\end{split}
\end{eqnarray}
\end{proposition}
The expectations in both the numerator and denominator of \eqref{eqn:Vbetaxstar} are taken with respect to the reference policy $\pi_{\mathrm{ref}}$. This makes it feasible to estimate $V_\beta^*(x)$ offline, without training or sampling from $\pi_\beta^*$ itself. Specifically, for each prompt, we sample a set of reference rollouts from $\pi_{\mathrm{ref}}$ and score them using the same automated verifier. We then compute $\widehat{V}_{\beta}(x)$, a plug-in estimate of \eqref{eqn:Vbetaxstar} by replacing the expectations in the numerator and denominator with their empirical averages over these reference rollouts. These offline estimates are then used as the initial value estimates in BASIS. Below, we describe how the hyperparameter $\beta$ is adaptively selected across training steps.

\begin{algorithm}[t]
\caption{One online training step of BASIS}
\label{alg:bas}
\begin{algorithmic}[1]
\Require Policy $\pi_{\theta_t}$, batch $\mathcal B=\{x_i\}_{i=1}^B$, search interval $\mathbb{B}$, cached reward statistics for evaluating $\widehat{V}_{\beta}(x_i)$, threshold $\epsilon$
\State Set $\theta_{\mathrm{old}} \gets \theta_t$
\For{$i=1,\dots,B$}
    \State Sample $y_i \sim \pi_{\theta_{\mathrm{old}}}(\bullet|x_i)$ and compute $r_i$
    \State Set $\widehat V_{i,t,\beta}\gets\widehat V_{\beta}(x_i)$, $\mathcal A_{t,\beta}\gets\{i:\epsilon <\widehat V_{i,t,\beta}<1-\epsilon\}$, compute $\widetilde V_{i,t,\beta}$ by Eqs.~\eqref{eq:linear-baseline} and \eqref{eqn:wij} for $\beta\in \mathbb B$
\EndFor
\State Select $\beta_t$ by Eq.~\eqref{eq:online-beta-selection}
\For{$i=1,\dots,B$}
    \State Set $b_i\gets \widetilde V_{i,t,\beta_t}\mathbbm{1}_{\{i \in \mathcal A_{t,\beta_t}\}}$ and $\widehat A_i\gets r_i-b_i$
\EndFor
\State Perform the policy update using $\{\widehat A_i\}_{i=1}^B$
\end{algorithmic}
\end{algorithm}

\smallskip

\noindent \textbf{Online calibration}. 
To illustrate how the KL-regularized optimal values $V_\beta^*$ can be used to approximate the values $V_{i,t}$ across training steps, consider two extreme cases:
\begin{enumerate}[leftmargin=*]
    \item At the first training step, the value $V_{i,t}$ is computed under the initial policy, which is essentially the reference policy. This corresponds to setting $\beta$ to infinity: the KL penalty becomes dominant, and the induced optimal policy $\pi_\beta^*$ coincides with the reference policy.
    \item When the algorithm converges, for sufficiently large $t$, $V_{i,t}$ approaches the value under the reward-optimal policy that maximizes \eqref{eqn:Jtheta}. This corresponds to setting $\beta$ close to zero: the KL penalty vanishes, and the induced optimal policy $\pi_\beta^*$ approaches the reward-optimal policy.
\end{enumerate}
More generally, the KL-regularized optimal values provide a continuum of value estimates that approximate how the policy value evolves from the reference policy toward a reward-optimal policy during training. Motivated by this observation, for an arbitrary training step $t$ and prompt $x$, we use the family of values $\{V_\beta^*(x):\beta\in\mathbb{B}\}$ to approximate the value of the current learning policy $V_{i,t}$. 

In our implementation, we set $\mathbb{B}=[0.01,5]$.  The calibration parameter $\beta$ is selected adaptively at each training step, allowing its value to vary as training progresses. Specifically, at each training step $t$, we select $\beta$ by minimizing
\begin{equation}\label{eq:online-beta-selection}
\beta_t \in \arg\min_{\beta\in\mathbb{B}}\;
\frac{1}{|\mathcal A_{t,\beta}|}
\sum_{i\in\mathcal A_{t,\beta}}
\bigl(r_i^{(t)}-\widetilde V_{i,t,\beta}\bigr)^2,
\end{equation}
where $\widetilde V_{i,t,\beta}$ and $\mathcal A_{t,\beta}$ denote, respectively, the batchwise refined estimator and the active set obtained by using $\widehat V_{\beta}(x_i)$ as the initial value estimate. 

Figure \ref{fig:basis-beta-traj} in Appendix \ref{app:online-implementation} visualizes the selected $\beta_t$. As expected, it generally decreases over training. 
Finally, we set the baseline $b_i=\widetilde V_{i,t,\beta_t}$ and compute the advantage as $\widehat A_{i,t}=r_i-b_i$.

\section{Experiments}\label{sec:experiments}
We conduct numerical experiments to evaluate BASIS along two dimensions: value estimation (Section \ref{sec:offline-diagnostics}), where we measure the accuracy of the estimated values, and policy optimization (Section \ref{sec:policyopt}), where we assess the downstream performance of the fine-tuned models. 

\subsection{Value estimation}\label{sec:offline-diagnostics}
We begin with a high-level summary of our findings and provide details in the following paragraphs. BASIS is highly sample efficient and robust for value estimation: 
    \begin{enumerate}[leftmargin=*]
        \item It reduces the REINFORCE++ value estimator's MSE by ${\bf 69}\%$, and with only a single rollout per prompt, achieves smaller MSE than GRPO-type estimators with {\bf 8} rollouts  (Figure \ref{fig:combined}(a)).
        \item Its estimator is also robust to reward heterogeneity within a batch, unlike the global baseline used by REINFORCE++ (Figure \ref{fig:combined}(b)), and remains robust across prompt difficulty levels, unlike GRPO-type algorithms (Figure \ref{fig:combined}(c)).
        \item Moreover, unlike GRPO, BASIS often produces an informative nonzero advantage (Figure \ref{fig:combined}(d)). 
    \end{enumerate}
    
\noindent \textbf{Setup}. We post-train Qwen2.5-Math-7B with GRPO on a subset of the MATH dataset consisting of Level 3--5 problems \citep{hendrycksmath2021}. We then use an intermediate training checkpoint to evaluate the accuracy of different value estimators.  

Specifically, for each prompt, we first approximate its oracle value by averaging rewards over 256 Monte Carlo rollouts from the checkpoint. We then repeat the following procedure: In each repeat, we independently sample a batch of prompts  from the same checkpoint, compute the BASIS, GRPO, RLOO and REINFORCE++ value estimators, and compute their squared errors against the Monte Carlo oracle values. We then average these squared errors over different repeats to obtain the MSE. More details are provided in Appendix~\ref{app:baseline-diagnostic-details}.
\begin{figure*}[t]
\centering
\includegraphics[width=\textwidth]{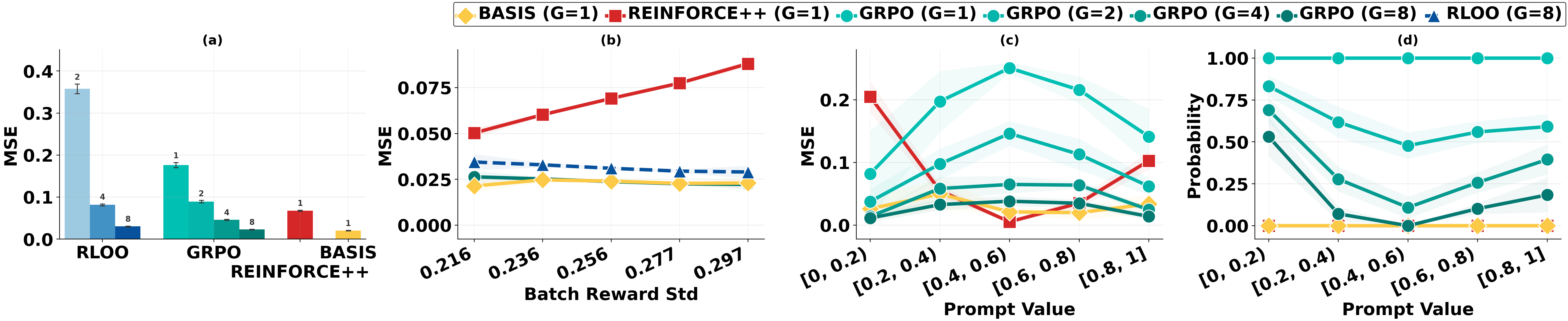}
\caption{MSE of various value estimates.  (a) MSE aggregated over all prompts. 
(b) MSE across batches grouped by the heterogeneity of within-group values, measured by the standard deviation of $\{V_{i,t}\}_{i=1}^B$. 
(c) MSE across prompts grouped by prompt-level value. 
(d) Frequency with which the estimated baseline is exactly $0$ or $1$. $G$ in the legend denotes the number of rollouts per prompt. In (b), we report GRPO and RLOO only with $8$ rollouts, since their MSEs are much larger with fewer rollouts. In (c) and (d), only MSEs of GRPO value estimators are reported to keep the figure readable. MSEs of RLOO value estimators are deferred to Figure~\ref{fig:combined_rloo} in Appendix~\ref{app:baseline-diagnostic-details}. In (d), the yellow line (BASIS) and the red line (REINFORCE++) largely overlap.}
\label{fig:combined}
\end{figure*}
\smallskip

\noindent \textbf{Finding 1: BASIS is sample efficient}. 
Figure \ref{fig:combined}(a) reports the MSE of different value estimators averaged over all prompts. Compared with single-rollout baselines such as REINFORCE++, BASIS reduces the MSE by ${\bf 69}\%$. More strikingly, even with one rollout per prompt, BASIS achieves lower MSE than GRPO and RLOO with ${\bf 8}$ rollouts. This demonstrates the utility of batchwise learning: by borrowing information across the entire batch, BASIS substantially improves value estimation without requiring repeated rollouts from the same prompt. Additionally, RLOO has larger MSE than GRPO, since its leave-one-out construction uses one fewer rollout for value estimation. 

\smallskip

\noindent \textbf{Finding 2: BASIS is robust}. 
Figure \ref{fig:combined}(b) reports the MSE of different value estimators across data batches grouped by the heterogeneity of their prompt-level values, measured by the standard deviation of $\{V_{i,t}\}_{i=1}^B$. Figure \ref{fig:combined}(c) reports the MSE across subsets of prompts grouped by their values. 

We make three observations. (i) First, BASIS is nearly flat across all five groups in both figures, demonstrating its robustness to both within-batch reward heterogeneity and prompt difficulty. (ii) Second, REINFORCE++ is sensitive to both factors. As batch-level heterogeneity increases from {\bf 0.21} to {\bf 0.31}, its MSE nearly doubles. This is expected because the variance of the global batch baseline used by REINFORCE++ increases naturally with the within-batch reward heterogeneity. Similarly, across prompt difficulty levels, REINFORCE++ performs best on medium-difficulty prompts but incurs much larger MSEs on easy and hard prompts, since the global batch mean pulls all baseline estimates toward the middle. (iii) Finally, GRPO and RLOO perform worst on medium-difficulty prompts. This is also expected for binary rewards: prompts with values near the middle have the largest reward variance.

\smallskip

\noindent \textbf{Finding 3: BASIS is informative}. 
Figure \ref{fig:combined}(d) reports how often the estimated baseline is exactly $0$ or $1$, leading to a zero advantage and therefore a zero contribution to the policy gradient estimator. BASIS never returns a $0$ or $1$ baseline in this setting, and therefore consistently provides an informative signal. In contrast, GRPO often produces such extreme baselines, particularly for easy and hard prompts. 

\subsection{Policy optimization}\label{sec:policyopt}
We now show that improved value and advantage estimation leads to more effective policy optimization. In particular, our experiments demonstrate three major findings:
\begin{enumerate}[leftmargin=*]
    \item The BASIS advantage estimator can be used as a versatile plug-in component for a number of multi-rollout GRPO-type algorithms, roughly {\bf halving} online learning time while preserving most of their downstream performance (Table~\ref{tab:qwen3-main}).
    \item Compared with single-rollout REINFORCE-type baselines, BASIS often achieves better downstream accuracy, with absolute improvements of up to {\bf 44.8} and {\bf 9.5} percentage points over REINFORCE and REINFORCE++,  respectively, while using roughly {\bf one-half} of their computation (Table \ref{tab:qwen3-main}). 
    \item REINFORCE can {\bf collapse} during training: its downstream accuracy may drop considerably as the number of training iterations increases. BASIS shows a much more stable learning curve, suggesting that more accurate value and advantage estimation stabilizes single-rollout policy optimization and helps avoid such collapse (see Figure~\ref{fig:qwen3-g1-vanilla-vs-basis} in Appendix~\ref{app:online-results}). 
\end{enumerate}

\noindent \textbf{Setup}. We compare different algorithms by using them to post-train Qwen3-4B on the DAPO-Math-17K training split \citep{yu2025dapo} and evaluating the resulting models on seven benchmarks: AIME 2024, AIME 2025, AMC 2023, MATH-500 \citep{hendrycksmath2021}, Minerva Math \citep{lewkowycz2022minerva}, OlympiadBench \citep{he2024olympiadbench}, and HMMT 2025 \citep{balunovic2026matharena}. For AIME, AMC, and HMMT, we report avg@32, the average accuracy of the resulting model over 32 sampled rollouts per problem. For MATH-500, Minerva Math, and OlympiadBench, we report accuracy from a single greedy deterministic rollout per problem. 

We also evaluate the models on five out-of-domain (OOD) benchmarks: 
HumanEval+ and MBPP+ for code generation \citep{liu2023evalplus}, MMLU-Pro for broad-domain knowledge and reasoning \citep{wang2024mmlupro}, MGSM for multilingual mathematical reasoning \citep{shi2022multilingual}, and IFEval for instruction following \citep{zhou2023ifeval}. We report each benchmark's standard task metric as well as their average.


To demonstrate the versatility of BASIS as a plug-in advantage estimation method, we combine it with the original GRPO and two representative follow-up variants: GPG \citep{chu2025gpg}, which removes the PPO step that requires importance sampling, and GSPO \citep{zheng2025gspo}, which replaces the per-token importance ratio with a sequence-level importance ratio. In each of these combinations, BASIS modifies only the advantage estimator: we use the BASIS baseline with a single rollout per prompt, and compare it against three baselines: the original GRPO-type algorithm, which uses eight rollouts per prompt and the group mean as the baseline; its vanilla single-rollout version, which uses a zero baseline; and REINFORCE++, which uses a global batch baseline. To ensure a fair comparison, all algorithms use the same sampling budget of $512$ rollouts per training step. Thus, the 8-rollout methods sample 64 prompts per training step, while the single-rollout methods sample 512 prompts per step. We train all baseline algorithms for $300$ steps, corresponding to approximately $1.1$ passes over the training set with $17{,}398$ prompts for 8-rollout methods and approximately $8.8$ passes for single-rollout methods. For BASIS, we train for only 150 steps, requiring roughly half the online learning time. As shown below, despite this reduced training budget, BASIS achieves performance comparable to the multi-rollout methods and often better performance than the single-rollout baselines.

\begin{table*}[t]
\centering
\small
\setlength{\tabcolsep}{4.0pt}
\renewcommand{\arraystretch}{1.10}
\resizebox{\textwidth}{!}{
\begin{tabular}{llccccccccc}
\toprule
\textbf{Objective} &
\textbf{Setting} &
\textbf{Time} &
\makecell{\textbf{AIME}\\\textbf{2024}} &
\makecell{\textbf{AIME}\\\textbf{2025}} &
\makecell{\textbf{AMC}\\\textbf{2023}} &
\makecell{\textbf{MATH}\\\textbf{-500}} &
\makecell{\textbf{Minerva}\\\textbf{Math}} &
\makecell{\textbf{Olympiad}\\\textbf{Bench}} &
\makecell{\textbf{HMMT}\\\textbf{2025}} &
\textbf{Avg} \\
\midrule

GRPO & \(G=8\), step 300 & 15.5h
& 0.312 & 0.298 & 0.716 & 0.880
& 0.467 & 0.537 & 0.088 & 0.471 \\
\rowcolor{VanillaShade}
GRPO Vanilla & \(G=1\), step 300 & 20.7h
& 0.023 & 0.021 & 0.331 & 0.444
& 0.279 & 0.234 & 0.015 & 0.192 \\
\rowcolor{VanillaShade}
\(\quad+\)REINFORCE++ & \(G=1\), step 300 & 17.7h
& 0.277 & 0.271 & 0.662 & 0.826
& 0.408 & 0.487 & 0.088 & 0.431 \\
\rowcolor{BASShade}
\(\quad+\)BASIS & \(G=1\), step 150 & \phantom{0}8.3h
& 0.303 & 0.281 & 0.758 & 0.892
& 0.426 & 0.559 & 0.093 & 0.473 \\
\rowcolor{BASShade}
& \(\Delta\) vs \(G=8\) & \dtime{-7.2h}
& \dnear{-0.8} & \dnear{-1.7} & \dpos{+4.2} & \dpos{+1.2}
& \dnear{-4.0} & \dpos{+2.2} & \dpos{+0.4} & \dpos{+0.2} \\
\rowcolor{BASShade}
& \(\Delta\) vs Vanilla & \dtime{-12.4h}
& \dpos{+28.0} & \dpos{+26.0} & \dpos{+42.7} & \dpos{+44.8}
& \dpos{+14.7} & \dpos{+32.5} & \dpos{+7.8} & \dpos{+28.1} \\
\rowcolor{BASShade}
& \(\Delta\) vs REINFORCE++ & \dtime{-9.4h}
& \dpos{+2.6} & \dpos{+1.0} & \dpos{+9.5} & \dpos{+6.6}
& \dpos{+1.8} & \dpos{+7.3} & \dpos{+0.4} & \dpos{+4.2} \\

\addlinespace[0.4em]

GPG & \(G=8\), step 300 & 14.0h
& 0.369 & 0.286 & 0.788 & 0.880
& 0.426 & 0.570 & 0.123 & 0.492 \\
\rowcolor{VanillaShade}
GPG Vanilla & \(G=1\), step 300 & 18.6h
& 0.056 & 0.093 & 0.558 & 0.488
& 0.184 & 0.182 & 0.025 & 0.227 \\
\rowcolor{VanillaShade}
\(\quad+\)REINFORCE++ & \(G=1\), step 300 & 12.9h
& 0.272 & 0.281 & 0.748 & 0.874
& 0.426 & 0.558 & 0.082 & 0.463 \\
\rowcolor{BASShade}
\(\quad+\)BASIS & \(G=1\), step 150 & \phantom{0}6.4h
& 0.338 & 0.284 & 0.762 & 0.884
& 0.430 & 0.574 & 0.107 & 0.483 \\
\rowcolor{BASShade}
& \(\Delta\) vs \(G=8\) & \dtime{-7.6h}
& \dnear{-3.0} & \dnear{-0.2} & \dnear{-2.7} & \dpos{+0.4}
& \dpos{+0.4} & \dpos{+0.5} & \dnear{-1.6} & \dnear{-0.9} \\
\rowcolor{BASShade}
& \(\Delta\) vs Vanilla & \dtime{-12.2h}
& \dpos{+28.2} & \dpos{+19.2} & \dpos{+20.4} & \dpos{+39.6}
& \dpos{+24.6} & \dpos{+39.2} & \dpos{+8.2} & \dpos{+25.6} \\
\rowcolor{BASShade}
& \(\Delta\) vs REINFORCE++ & \dtime{-6.5h}
& \dpos{+6.7} & \dpos{+0.3} & \dpos{+1.4} & \dpos{+1.0}
& \dpos{+0.4} & \dpos{+1.6} & \dpos{+2.5} & \dpos{+2.0} \\

\addlinespace[0.4em]

GSPO & \(G=8\), step 300 & 14.4h
& 0.365 & 0.296 & 0.794 & 0.896
& 0.463 & 0.588 & 0.117 & 0.502 \\
\rowcolor{VanillaShade}
GSPO Vanilla & \(G=1\), step 300 & 13.0h
& 0.234 & 0.241 & 0.751 & 0.868
& 0.423 & 0.586 & 0.105 & 0.458 \\
\rowcolor{VanillaShade}
\(\quad+\)REINFORCE++ & \(G=1\), step 300 & 13.0h
& 0.350 & 0.265 & 0.794 & 0.908
& 0.441 & 0.599 & 0.108 & 0.495 \\
\rowcolor{BASShade}
\(\quad+\)BASIS & \(G=1\), step 150 & \phantom{0}7.2h
& 0.320 & 0.269 & 0.754 & 0.888
& 0.460 & 0.565 & 0.114 & 0.481 \\
\rowcolor{BASShade}
& \(\Delta\) vs \(G=8\) & \dtime{-7.2h}
& \dnear{-4.5} & \dnear{-2.7} & \dnear{-4.0} & \dnear{-0.8}
& \dnear{-0.4} & \dnear{-2.2} & \dnear{-0.3} & \dnear{-2.1} \\
\rowcolor{BASShade}
& \(\Delta\) vs Vanilla & \dtime{-5.8h}
& \dpos{+8.5} & \dpos{+2.8} & \dpos{+0.3} & \dpos{+2.0}
& \dpos{+3.7} & \dnear{-2.1} & \dpos{+0.8} & \dpos{+2.3} \\
\rowcolor{BASShade}
& \(\Delta\) vs REINFORCE++ & \dtime{-5.8h}
& \dnear{-3.0} & \dpos{+0.4} & \dnear{-4.0} & \dnear{-2.0}
& \dpos{+1.8} & \dnear{-3.4} & \dpos{+0.5} & \dnear{-1.4} \\

\bottomrule
\end{tabular}}
\caption{Online policy optimization experiments on Qwen3-4B. $G$ denotes the number of rollouts per prompt. BASIS with \(G{=}1\) is trained for $150$ steps and compared against three baselines trained for $300$ steps: (i) the $8$-rollout (\(G{=}8\)) GRPO-type algorithm, (ii) its vanilla single-rollout (\(G{=}1\)) variant, and (iii) another single-rollout variant combined with the REINFORCE++ global batch baseline. \textbf{Avg} is the unweighted average accuracy across the seven math benchmarks shown. Each BASIS row is followed by three \(\Delta\) rows comparing it against the three baselines, respectively. Accuracy differences are reported in percentage points. The \(\Delta\) entry under \emph{Time} gives the reduction in online learning wall-clock time achieved by BASIS. Green indicates that BASIS outperforms the baseline; gray indicates that BASIS is within 5 percentage points below the baseline.}
\label{tab:qwen3-main}
\end{table*}

For the OOD evaluation, we apply BASIS only to GRPO. Additional implementation details as well as the computational cost of BASIS for baseline estimation are provided in Appendix~\ref{app:online-implementation}. In Appendix~\ref{app:online-results}, we report an additional study in which we post-train Qwen2.5-Math-7B and compare BASIS against GRPO under the same protocol. 

\begin{table*}[t]
\centering
\small
\setlength{\tabcolsep}{4.0pt}
\renewcommand{\arraystretch}{1.10}
\resizebox{\textwidth}{!}{
\begin{tabular}{llcccccc}
\toprule
\textbf{Objective} &
\textbf{Setting} &
\textbf{HumanEval+} &
\textbf{MBPP+} &
\textbf{MMLU-Pro} &
\textbf{MGSM} &
\textbf{IFEval} &
\textbf{Avg} \\
\midrule

Base model & no RL
& 67.1 & 64.0 & 51.2 & 78.9 & 79.7 & 68.2 \\

\addlinespace[0.2em]

GRPO & \(G=8\), step 300
& 82.3 & 71.2 & 61.5 & 82.0 & 83.4 & 76.1 \\

\rowcolor{VanillaShade}
GRPO Vanilla & \(G=1\), step 300
& 57.9 & 53.4 & 47.8 & 69.1 & 76.3 & 60.9 \\

\rowcolor{VanillaShade}
\(\quad+\)REINFORCE++ & \(G=1\), step 300
& 73.8 & 66.9 & 56.3 & 80.5 & 78.9 & 71.3 \\

\rowcolor{BASShade}
\(\quad+\)BASIS & \(G=1\), step 150
& 82.3 & 69.8 & 61.8 & 82.6 & 80.2 & 75.3 \\

\rowcolor{BASShade}
& \(\Delta\) vs \(G=8\)
& \dnear{0.0} & \dnear{-1.4} & \dpos{+0.3} & \dpos{+0.6}
& \dnear{-3.2} & \dnear{-0.7} \\

\rowcolor{BASShade}
& \(\Delta\) vs Vanilla
& \dpos{+24.4} & \dpos{+16.4} & \dpos{+14.0} & \dpos{+13.5}
& \dpos{+3.9} & \dpos{+14.4} \\

\rowcolor{BASShade}
& \(\Delta\) vs REINFORCE++
& \dpos{+8.5} & \dpos{+2.9} & \dpos{+5.5} & \dpos{+2.1}
& \dpos{+1.3} & \dpos{+4.1} \\

\bottomrule
\end{tabular}}
\caption{OOD evaluation of GRPO variants trained on DAPO-Math-17K.}
\label{tab:ood-eval}
\end{table*}

\smallskip

\noindent \textbf{Finding 1. BASIS is competitive with 8-rollout methods while halving the online learning budget.} 
Table~\ref{tab:qwen3-main} reports the results on the seven in-domain mathematical reasoning benchmarks. Compared with 8-rollout GRPO, BASIS achieves a slightly higher average accuracy across the seven benchmarks, with a gain of {\bf 0.2} percentage points, and outperforms GRPO on {\bf 4} of the seven benchmarks. Compared with 8-rollout GPG and GSPO, BASIS is slightly lower on average, by {\bf 0.9} and {\bf 2.1} percentage points, respectively. Importantly, these results are obtained with half the online sampling budget: BASIS uses {\bf 76.8}K sampled responses, compared with {\bf 153.6}K for the 8-rollout baselines, reducing online learning wall-clock time by {\bf 7.2} to {\bf 7.6} hours across the three GRPO-type algorithms.

\smallskip 

\noindent \textbf{Finding 2. BASIS often outperforms single-rollout methods using roughly one-half of their online learning time}.
Table~\ref{tab:qwen3-main} further shows that BASIS improves the average accuracy across the seven benchmarks by {\bf 28.1} and {\bf 25.6} percentage points over single-rollout vanilla GRPO and GPG, respectively, and by {\bf 4.2} and {\bf 2.0} percentage points over their REINFORCE++ variants. These gains are positive on every individual benchmark, ranging from {\bf 0.3} to {\bf 44.8} percentage points. Importantly, BASIS achieves these improvements with substantially less training time, ranging from about {\bf half} the compute used by the REINFORCE++ variants to about {one-third} of that used by single-rollout GPG. 

Compared with single-rollout GSPO, the improvement is more modest. Using roughly {\bf 55\%} of the training time of the two GSPO single-rollout variants, BASIS improves the average accuracy over vanilla GSPO by {\bf 2.3} percentage points, but is {\bf 1.4} percentage points lower than its REINFORCE++ variant with half of the compute budget.

The OOD results reported in Table~\ref{tab:ood-eval} further confirm this finding: BASIS improves the average score by {\bf 14.4} percentage points over vanilla single-rollout GRPO and by {\bf 4.1} points over REINFORCE++, with consistent gains across all five OOD benchmarks.

\smallskip
\noindent \textbf{Finding 3. BASIS prevents collapse in single-rollout policy optimization}.  
A closer look into the learning curves in Figure~\ref{fig:qwen3-g1-vanilla-vs-basis} (Appendix~\ref{app:online-results}) reveals that vanilla GRPO and GPG collapse during training: vanilla GRPO barely learns, and its accuracy measure generally decreases over training. Vanilla GPG improves at first, but after peaking before the middle of training, its performance steadily declines. In contrast, BASIS does not suffer from this collapse. Together with the results in Section~\ref{sec:offline-diagnostics}, these findings suggest that sample-efficient value and advantage estimation stabilizes single-rollout policy optimization and helps prevent collapse.

\section{Discussion}
This paper introduces BASIS for rollout-efficient RLVR. Methodologically, BASIS samples only one rollout per prompt during online training, while leveraging batchwise information sharing to improve value and advantage estimation (Figure \ref{fig:online-weighting-cached-value}). Empirically, BASIS is highly sample efficient and robust for value function estimation, and often produces informative advantage estimates (Figure~\ref{fig:combined}). For policy optimization, BASIS often improves over single-rollout baselines, helps prevent collapse during training and achieves similar performance to multi-rollout baselines with substantially less computation (Table~\ref{tab:qwen3-main}). These gains also persist in out-of-domain evaluations covering code generation, general knowledge, multilingual reasoning, and instruction following (Table~\ref{tab:ood-eval}). 

Although our experiments in Section~\ref{sec:experiments} primarily use binary verifiable rewards, BASIS can also be extended to settings with continuous rewards. Appendix~\ref{app:continuous-rewards} discusses this extension and evaluates it in a reinforcement learning from human feedback (RLHF) setting on the Helpful and Harmless (HH) dataset \citep{bai2022hh}. It can be seen from Figure~\ref{fig:hh-rlhf-continuous} that BASIS attains higher training rewards than the single-rollout REINFORCE and REINFORCE++ baselines in the later stage of training, providing evidence that its benefits extend beyond binary reward RLVR settings.

\section*{Limitations}
This work has a few limitations. First, the BLUE and unbiasedness guarantees in Proposition~\ref{proBLUB} require the initial value estimates to be independent of the data batch used for batchwise refinement. These guarantees therefore do not directly apply when the initial estimates are adaptively calibrated using the same batch. A simple way to preserve the required independence is to calibrate the initial value estimates using the preceding batch instead. We detail this variant in Appendix~\ref{app:online-implementation}. 
Second, BASIS improves value and advantage estimation by borrowing information across prompts within the same training batch. A complementary line of work borrows information for the same prompt across different training steps \citep{wang2025kalman,xu2025single,gong2026kernelized,ruan2026spo++}. These approaches could potentially be combined with BASIS to share information both across prompts within a batch and across training iterations, further improving the post-training algorithm's sample efficiency. 

\section*{Acknowledgements}
This work was supported by the UK Government's AI Research Resource (AIRR) programme and undertaken using the Isambard-AI system.

\bibliographystyle{apalike}
\bibliography{reference}

\clearpage
\appendix

\section*{Appendix}

\section{Experimental Details for Value Estimation}
\label{app:baseline-diagnostic-details}

This section details the value estimation experiments  in
Section~\ref{sec:offline-diagnostics}.  All experiments are conducted at a fixed
GRPO-trained checkpoint \(\pi_{\theta_t}\), initialized from Qwen2.5-Math-7B and optimized on the Level 3--5 subset of the MATH dataset \citep{hendrycksmath2021}. For each repeat, we sample a batch of prompts and estimate each prompt's oracle value \(V_{i,t}\) using 256 Monte Carlo rollouts from the checkpoint. All value estimators are then constructed from a separate set of reward samples from the same checkpoint. 

More specifically, we fix the batch size at \(B=64\) and vary the number of rollouts per prompt \(G\in\{1,2,4,8\}\) for GRPO and \(G\in\{2,4,8\}\) for RLOO (which requires at least two rollouts to leave one out), while BASIS is evaluated in the single-rollout setting with \(G=1\).   For each repeat, we draw a batch of prompts without replacement and sample up to eight rewards per prompt.  To ensure a fair comparison across different values of $G$, the value estimator for each $G$ is computed using the first $G$ rewards from the same set of rollouts.   The reported metric is the mean squared error \((\widehat V_{i,t}-V_{i,t})^2\), averaged over prompts and then over 10 repeats.  

\paragraph{Baseline algorithms.}
We compare BASIS with three baseline algorithms:
\begin{itemize}[leftmargin=*]
    \item \textbf{GRPO} uses the within-prompt group reward mean as its baseline:
    \[
    b^{\mathrm{GRPO}}_{i,g}
    =
    \bar r_i
    :=
    \frac{1}{G}\sum_{\ell=1}^{G} r_{i,\ell}
    \]

    \item \textbf{RLOO} uses a leave-one-out group baseline, where the baseline
    for each sample is computed from the other completions generated from the
    same prompt:
    \[
    b^{\mathrm{RLOO}}_{i,g}
    =
    \bar r_{i,-g}
    :=
    \frac{1}{G-1}\sum_{\ell\ne g} r_{i,\ell},
    \]

    \item \textbf{REINFORCE++} uses the global reward mean within the current batch
    as its baseline:
    \[
     b^{\mathrm{R++}}_{i}
    =
    \bar r_{\mathrm{batch}}
    :=
    \frac{1}{B}\sum_{j=1}^{B}r_{j},
    \]
\end{itemize}


\paragraph{Effect of batch heterogeneity.}
We examine how the accuracy of the value estimator varies with the heterogeneity of the sampled prompt batch. We fix \(B=64\) and sample 500 batches to ensure sufficient diversity across the sampled prompt batches. For each batch \(b\), we define its heterogeneity score as
$s_b:=\operatorname{std}_{i\in b}(V_{i,t}),$
which is the standard deviation of the 64 oracle prompt values in that batch.
The minimum and maximum observed heterogeneity scores across the 500 batches are \(0.206\) and \(0.307\), and we divide this interval into five bins with uniformly spaced boundaries: 
\([0.206,0.226)\), \([0.226,0.246)\), \([0.246,0.266)\),
\([0.266,0.287)\), and \([0.287,0.307]\). Within each bin, we report the MSE of each value estimator, first averaged over prompts within each batch and then averaged across batches in the bin.

\paragraph{Effect of prompt difficulty.}
We further evaluate how the accuracy of the value estimator varies across different prompt-difficulty levels. For each repeat, we first sample a batch of \(B=64\) prompts and compute the oracle value estimate per prompt. We next stratify all prompts by their oracle values into five difficulty bins: \([0,0.2)\), \([0.2,0.4)\), \([0.4,0.6)\), \([0.6,0.8)\), and \([0.8,1]\). Within each bin, we similarly report the MSE of each value estimator, averaged over prompts within that bin and across 10 repeats.

\begin{figure}[t]
\centering
\includegraphics[width=0.48\textwidth]{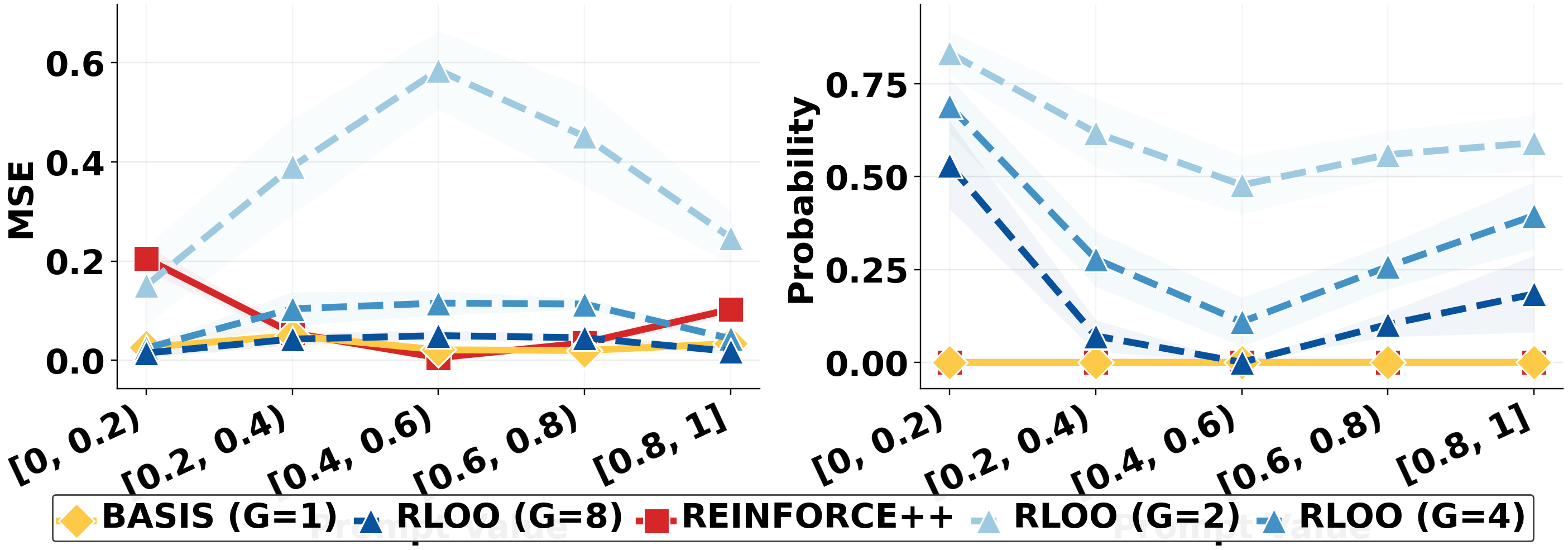}
\caption{MSE of RLOO value estimator (left) and frequency with which its estimated baseline is exactly 0 or 1 (right). Results for BASIS and REINFORCE++ are also reported for completeness.}
\label{fig:combined_rloo}
\end{figure}

\paragraph{Sensitivity analysis.} We evaluate the sensitivity of BASIS to the active-set threshold $\epsilon$, the number of offline reference samples $n$, the online estimation batch size $B$, and the number of candidate values in the $\beta$ grid. Following the value estimation experiments in Section~\ref{sec:offline-diagnostics}, we use Qwen2.5-Math-7B checkpoints from the early, middle, and late stages of training and compute oracle values using 256 held-out rollouts. We vary one hyperparameter at a time while holding all others at their default values ($\epsilon=10^{-6}$, $n=64$, and $B=64$, together with the default $\beta$ grid). 

\begin{figure*}[t]
\centering
\includegraphics[width=\textwidth]{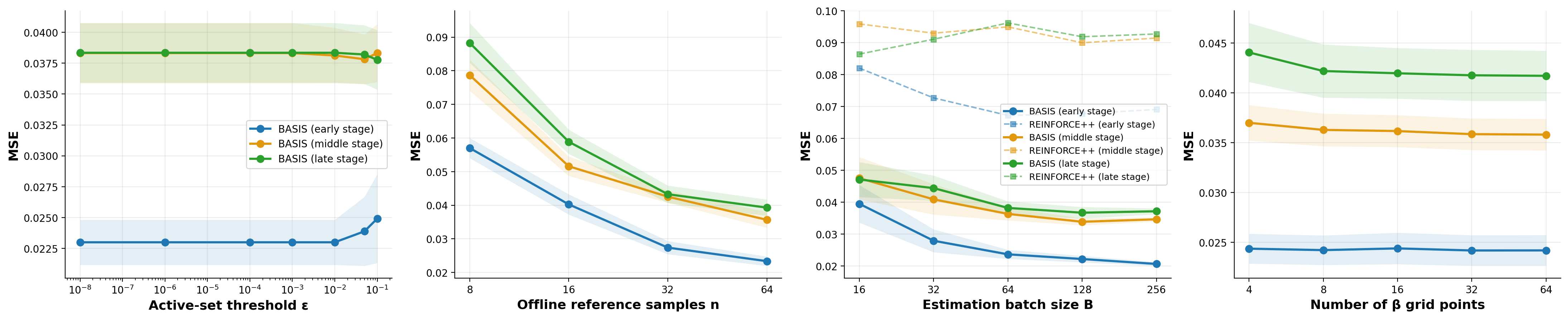}
\caption{MSE of BASIS value estimator at early, middle, and late training stages in the sensitivity analysis. Shaded regions are 95\% confidence intervals. The batch-size panel also reports the performance of REINFORCE++ for reference.}
\label{fig:basis-sensitivity}
\end{figure*}

Figure~\ref{fig:basis-sensitivity} shows that BASIS is stable over a broad range of hyperparameter choices. The MSE is nearly unchanged for $\epsilon\in[10^{-8},10^{-2}]$, and using between 8 and 64 candidate $\beta$ values produces similar results. Increasing the offline cache from $n=8$ to $n=64$ reduces MSE by approximately 55--59\% across the three training stages. Increasing $B$ improves estimation up to about $B=64$--$128$, beyond which the improvement is marginal.

\section{Alternative Information-Sharing Approaches}
\label{app:supplementary-baselines}

In Section~\ref{sec:method-section}, we construct the BASIS baseline as a weighted average of rewards from the same batch, with data-adaptive weights chosen to form a best linear unbiased baseline; we refer to this main estimator as UNB for unbiasedness. More broadly, batchwise information can be shared in other ways, such as using weighted averages without the unbiasedness constraint, or using importance weights based on ratios of initial value estimates. In this section, we study two such alternatives, namely, the Variance-OPtimized shrinkage baseline (VOP) and the Ratio-average Value-Guided baseline (RVG). Using the Qwen2.5-Math-7B checkpoints, single-rollout VOP and RVG perform comparably to UNB and match or outperform GRPO with $G{=}8$ rollouts on most math benchmarks, while using roughly half the online training time (see Appendix~\ref{app:online-results}, Table~\ref{tab:qwen25-main}).

Before defining each variant, we note that both approaches reuse the same offline initial value estimator and the same calibration method for selecting $\beta_t$ as UNB does. We keep similar notations from
Section~\ref{sec:method-section}:
$\widehat V_{i,t}=\widehat V_{\beta_t}(x_i)$,
$\widehat\sigma_i^2=\widehat V_{i,t}(1-\widehat V_{i,t})$, and
$\mathcal A_t$ is the active set in Algorithm~\ref{alg:bas}. The
formulas below apply only to active prompts; prompts outside
$\mathcal A_t$ use the zero baseline.
For reference, the UNB proposed in the main text can be written as
\begin{equation}
\widetilde V_{i,t}^{\mathrm{UNB}}
=
\widehat V_{i,t}\,
\frac{\displaystyle
\sum_{j\in\mathcal A_t\setminus\{i\}}
\widehat V_{j,t} r_j/\widehat\sigma_j^2}
{\displaystyle
\sum_{j\in\mathcal A_t\setminus\{i\}}
\widehat V_{j,t}^2/\widehat\sigma_j^2},
\qquad i\in\mathcal A_t,
\label{eq:app-unb-baseline}
\end{equation}
by plugging $\widehat V_{i,t}$ and $\widehat\sigma_i^2$ into
Eq.~\eqref{eqn:wij}. 

Both variants detailed below share the same idea of weighting the batch rewards by the offline values $\widehat{V}_{i,t}$, and they differ in the weights.

\paragraph{Variance-OPtimized shrinkage baseline (VOP).}
VOP drops the unbiasedness constraint adopted by UNB. The following proposition derives VOP's weights; its proof is given in Appendix~\ref{app:proof}.

\begin{proposition}\label{prop:vop}
Fix a training step \(t\), a target prompt \(i\), and condition on the
prompt batch \(\mathcal X_B:=\{x_k\}_{k=1}^B\). Among all leave-one-out linear baselines
\(b_i=\sum_{j\ne i}w_{ij}r_j\), without imposing the unbiasedness
constraint, the weights minimizing
\(\mathbb E[(V_{i,t}-b_i)^2\mid\mathcal X_B]\) are
\[
w_{ij}^{\mathrm{VOP}}
=
\frac{V_{i,t}V_{j,t}/\sigma_j^2}
{1+\sum_{k\ne i}V_{k,t}^2/\sigma_k^2},
\qquad j\ne i .
\]
\end{proposition}

By Proposition~\ref{prop:vop}, for each active prompt $i$, we define the VOP estimator by
\begin{equation*}
\widetilde V_{i,t}^{\mathrm{VOP}}
=
\widehat V_{i,t}\,
\frac{\displaystyle\sum_{j\in\mathcal A_t\setminus\{i\}}\widehat V_{j,t} r_j/\widehat\sigma_j^2}
{\displaystyle 1+\sum_{j\in\mathcal A_t\setminus\{i\}}\widehat V_{j,t}^2/\widehat\sigma_j^2},
\quad i\in\mathcal A_t .
\end{equation*}
Compared to the UNB estimator in Eq.~\eqref{eq:app-unb-baseline}, VOP includes an additional $1$ in the denominator, which shrinks the estimator toward zero, trading unbiasedness for lower MSE.

\paragraph{Ratio-average Value-Guided baseline (RVG).}
RVG computes the value estimator by 
\begin{equation*}
\widetilde V_{i,t}^{\mathrm{RVG}}
=
\frac{1}{|\mathcal A_t|-1}
\sum_{j\in\mathcal A_t\setminus\{i\}}
\frac{\widehat V_{i,t}}{\widehat V_{j,t}}r_j,
\quad i\in\mathcal A_t.
\end{equation*}
It can be viewed as an importance sampling estimator, where the ratio $\widehat V_{i,t}/\widehat V_{j,t}$ reweights rewards from other prompts to borrow information. 

%

\section{Additional Online Learning Results}
\label{app:online-results}



\begin{table}[h]
\centering
\small
\setlength{\tabcolsep}{4pt}
\renewcommand{\arraystretch}{1.1}
\begin{tabular}{llccc}
\toprule
\textbf{Objective} & \textbf{Setting} & \textbf{Step} & \textbf{Responses} & \textbf{Time} \\
\midrule
GRPO              & \(G=8\) & 300 & 153.6K          & 15.5h \\
GRPO Vanilla      & \(G=1\) & 300 & 153.6K          & 20.7h \\
\(\quad+\)REINFORCE++ & \(G=1\) & 300 & 153.6K      & 17.7h \\
\(\quad+\)BASIS   & \(G=1\) & 150 & \phantom{0}76.8K & \phantom{0}8.3h \\
\(\quad+\)BASIS   & \(G=1\) & 300 & 153.6K          & 16.6h \\
\addlinespace[0.2em]
GPG               & \(G=8\) & 300 & 153.6K          & 14.0h \\
GPG Vanilla       & \(G=1\) & 300 & 153.6K          & 18.6h \\
\(\quad+\)REINFORCE++ & \(G=1\) & 300 & 153.6K      & 12.9h \\
\(\quad+\)BASIS   & \(G=1\) & 150 & \phantom{0}76.8K & \phantom{0}6.4h \\
\(\quad+\)BASIS   & \(G=1\) & 300 & 153.6K          & 12.7h \\
\addlinespace[0.2em]
GSPO              & \(G=8\) & 300 & 153.6K          & 14.4h \\
GSPO Vanilla      & \(G=1\) & 300 & 153.6K          & 13.0h \\
\(\quad+\)REINFORCE++ & \(G=1\) & 300 & 153.6K      & 13.0h \\
\(\quad+\)BASIS   & \(G=1\) & 150 & \phantom{0}76.8K & \phantom{0}7.2h \\
\(\quad+\)BASIS   & \(G=1\) & 300 & 153.6K          & 14.3h \\
\bottomrule
\end{tabular}
\caption{Computation time for fine-tuning Qwen3-4B models. Both single-rollout and multi-rollout  baseline algorithms are trained for $300$ steps, corresponding to approximately $1.1$ passes over the DAPO-Math-17K training set for 8-rollout methods and approximately $8.8$ passes for single-rollout methods. BASIS is trained for $150$ steps.}
\label{tab:qwen3-compute}
\end{table}

\begin{figure*}[t]
\centering
\includegraphics[width=\textwidth]{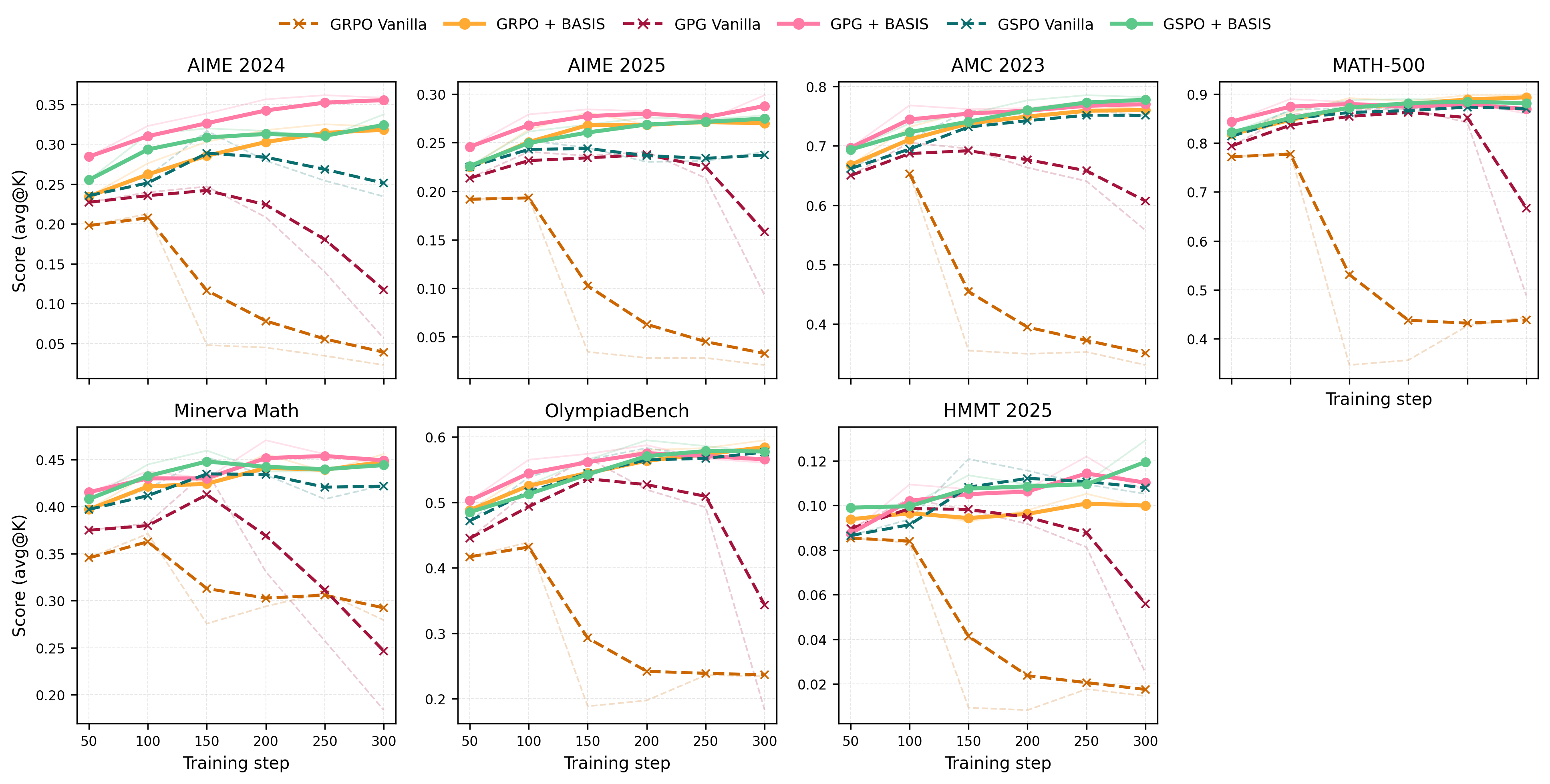}
\caption{Accuracy of Qwen3-4B models fine-tuned by 
\textit{Vanilla} (with a zero baseline) and\ \textit{BASIS} coupled with GRPO (orange), GPG (red), GSPO (teal) across seven math benchmarks. Solid lines are exponentially smoothed with a window of $3$; lighter curves in the same colors show the raw scores.}
\label{fig:qwen3-g1-vanilla-vs-basis}
\end{figure*}

Section \ref{sec:experiments} reports BASIS's performance at step 150. In this section, we evaluate its performance over steps $150$--$300$ and compare against the vanilla single-rollout algorithms. The results are visualized in 
Figure~\ref{fig:qwen3-g1-vanilla-vs-basis}. We observe that vanilla GRPO and GPG collapse after approximately step 150, with their performance deteriorating substantially as training progresses. 
Only vanilla GSPO survives beyond this point. BASIS lifts every objective onto a
monotonically improving trajectory. 
Figure~\ref{fig:qwen3-g1-rfpp-vs-basis} compares BASIS against a stronger single-rollout baseline, 
REINFORCE++. It can be seen that BASIS
dominates uniformly across all seven benchmarks, with a visible gap from Step $100$ onward. 


\begin{figure*}[t]
\centering
\includegraphics[width=\textwidth]{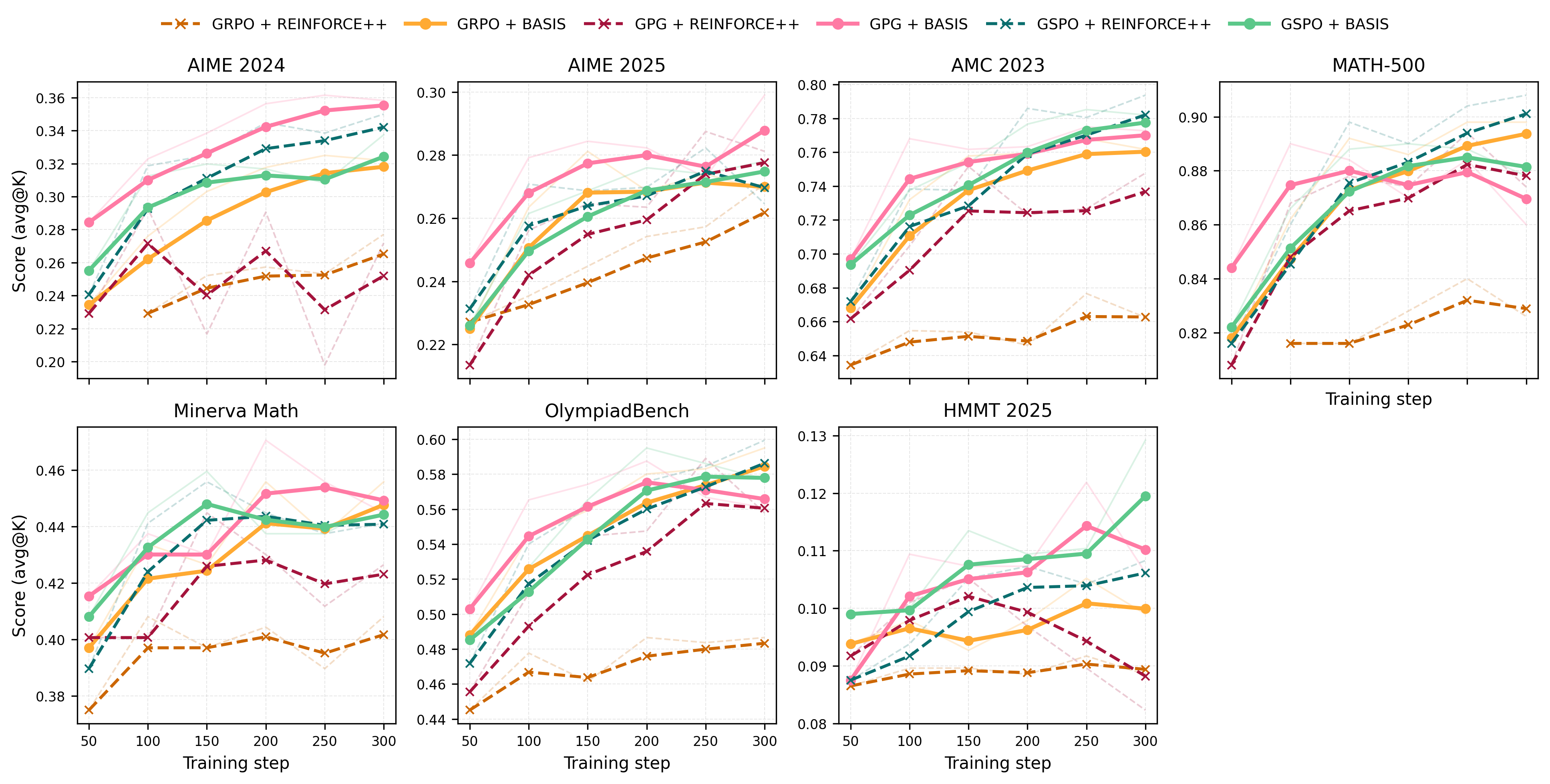}
\caption{Accuracy of Qwen3-4B models fine-tuned by 
\textit{REINFORCE++} (with a global batch baseline) and \textit{BASIS} coupled with GRPO (orange), GPG (red), GSPO (teal) across seven math benchmarks. 
The rest is the same as in Figure~\ref{fig:qwen3-g1-vanilla-vs-basis}.}
\label{fig:qwen3-g1-rfpp-vs-basis}
\end{figure*}

In addition to the experiments on Qwen3-4B, we apply the same online RL protocol to Qwen2.5-Math-7B, compare the three BASIS weighting variants described in Appendix  \ref{app:supplementary-baselines} and report the results in Table \ref{tab:qwen25-main}. We observe that the three BASIS variants match or outperform GRPO on five of the seven benchmarks, despite being trained for only 150 steps, compared with 300 steps for GRPO: UNB performs best on AMC~2023 and MATH-500, RVG on AIME~2024 and OlympiadBench, and VOP on AIME~2025. These results demonstrate that BASIS generalizes across models (Qwen2.5-Math-7B and Qwen3-4B) and across different weighting variants, suggesting that the benefits of information sharing are not tied to a particular choice of weights and models.

\begin{table}[h]
\centering
\small
\setlength{\tabcolsep}{4pt}
\renewcommand{\arraystretch}{1.30}
\resizebox{\columnwidth}{!}{
\begin{tabular}{lccccccc}
\toprule
\textbf{Method} &
\makecell{\textbf{AIME}\\\textbf{2024}} &
\makecell{\textbf{AIME}\\\textbf{2025}} &
\makecell{\textbf{AMC}\\\textbf{2023}} &
\makecell{\textbf{MATH}\\\textbf{-500}} &
\makecell{\textbf{Minerva}\\\textbf{Math}} &
\makecell{\textbf{Olympiad}\\\textbf{Bench}} &
\makecell{\textbf{HMMT}\\\textbf{2025}} \\
\midrule
GRPO \((G=8)\)
& 0.244 & 0.102 & 0.597 & 0.762
& \best{0.324} & 0.406 & \best{0.019} \\
\rowcolor{BASShade}
\(\quad+\)BASIS-UNB \((G=1)\)
& \second{0.245} & 0.107 & \best{0.655} & \best{0.780}
& \second{0.309} & 0.409 & \best{0.019} \\
\rowcolor{BASShade}
\(\quad+\)BASIS-RVG \((G=1)\)
& \best{0.256} & \second{0.117} & 0.622 & \second{0.773}
& 0.298 & \best{0.417} & 0.015 \\
\rowcolor{BASShade}
\(\quad+\)BASIS-VOP \((G=1)\)
& 0.240 & \best{0.118} & \second{0.634} & 0.768
& 0.303 & \second{0.416} & \second{0.018} \\
\bottomrule
\end{tabular}}
\caption{Online policy optimization experiments on Qwen2.5-Math-7B. Similar to the experiments on Qwen3-4B, BASIS is trained for 150
steps whereas GRPO is trained for 300 steps. All algorithms sample $512$ rollouts per training step, so BASIS methods train
on eight times as many different prompts per step as 8-rollout GRPO. UNB is the main algorithm presented in the main paper (Table~\ref{tab:qwen3-main}); RVG and VOP variants are defined in
Appendix~\ref{app:supplementary-baselines}. Best result in each column is
bolded, and second best is underlined.}
\label{tab:qwen25-main}
\end{table}


\section{Implementation Details for Policy Optimization}
\label{app:online-implementation}

Before online RL, we sample $n$ completions from the reference
policy on each training prompt and score them with the same verifier
used for online rewards. For each prompt $x_i$, the cache stores the
offline sample size $n_i = n$ and the empirical mean reward $\hat p_i$. We
set $n=64$ for the Qwen3-4B experiments and $n=32$ for the Qwen2.5-Math-7B experiments. In general, we recommend using a comparable offline sample size, subject to the available offline sampling budget. 

For binary rewards, our estimated value takes the following form,
\[
\widehat V_\beta(x_i) = \frac{\hat p_i \exp({1/\beta})}{1-\hat p_i + \hat p_i \exp({\,1/\beta})},
\]
given a regularization parameter $\beta$. We evaluate $\widehat V_\beta(x_i)$ on a fixed grid of $230$
values of $\beta$ spanning $[0.01, 5.0]$: 200 values with a step size of $0.01$ over $[0.01, 2.00]$, followed by 30 values with a step size of $0.1$ over $[2.1, 5.0]$. This yields a $17{,}398 \times 230$ lookup table, occupying approximately $16$ MB on disk. Figure~\ref{fig:basis-beta-traj} visualizes the resulting calibrated $\beta$ values over the course of online training.

The offline rollout process takes a few hours on the same four-GPU node used for training, while the subsequent closed-form evaluation of $\widehat V_\beta$ takes only seconds on a single CPU. Once computed, the initial estimator can be reused across all three policy objectives and all three weighting variants (UNB / RVG / VOP). Since this precomputation is shared across the BASIS settings, its effective per-experiment overhead is well below $5\%$ of the training wall-clock time.

\begin{figure}[t]
\centering
\includegraphics[width=\columnwidth]{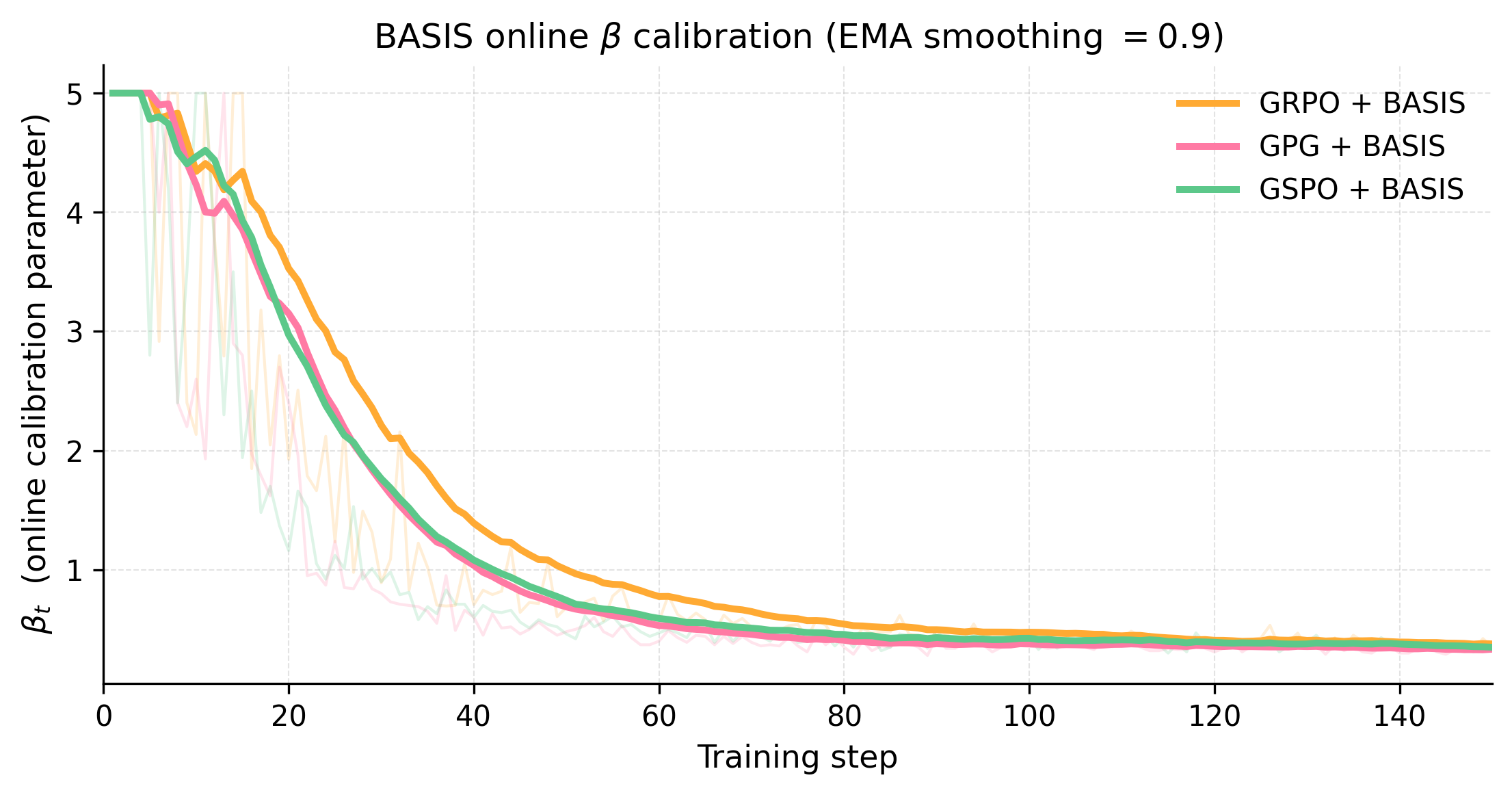}
\caption{Calibrated values of $\beta_t$ during BASIS training on Qwen3-4B when coupled with GRPO, GPG, and GSPO. At each training step the
calibrator selects $\beta_t$ by minimizing
Eq.~\eqref{eq:online-beta-selection} over the precomputed $230$-point
grid $\beta\in[0.01, 5.0]$. All three runs start near the upper end of the grid and settle into a narrow band around $\beta\approx 0.4$ by approximately step 80, where they remain for the rest of training. Solid lines show the smoothed calibrated $\beta_t$ values; lighter curves in the same colors show the corresponding raw  selected $\beta_t$ values.}
\label{fig:basis-beta-traj}
\end{figure}

For online policy learning, we adopt the implementation from \texttt{verl}~\citep{sheng2024hybridflow}, replacing only its advantage estimator with BASIS's estimator. 
At each calibration step, BASIS chooses $\beta_t$ by minimizing
Eq.~\eqref{eq:online-beta-selection} over the precomputed grid. This amounts to minimizing over the 230 precomputed grid values for each of the 512 active prompts, costing less than a millisecond per step on a single GPU. 
As an alternative to using the same data batch for both batchwise refinement and calibration, one can select the calibration parameter using the preceding online batch to preserve the unbiasedness of the value estimate: 
\begin{align*}
\mathcal L_s(\beta)
&=\frac{1}{|\mathcal A_{s,\beta}|}
\sum_{i\in\mathcal A_{s,\beta}}
\bigl(r_i^{(s)}-\widetilde V_{i,s,\beta}\bigr)^2,\\
\beta_t^{\mathrm{PB}}
&\in \arg\min_{\beta\in\mathbb B}\mathcal L_{t-1}(\beta),
\qquad
b_i^{(t)}=\widetilde V_{i,t,\beta_t^{\mathrm{PB}}}.
\end{align*}
Prompts whose selected offline value falls outside the active set use the zero baseline; all other prompts use the cross-prompt BASIS baseline. We set $\epsilon = 10^{-6}$ for the active set threshold throughout. 

Finally, the Qwen2.5-Math-7B experiments use max prompt length $1024$, max
response length $2048$, learning rate $10^{-6}$, rollout temperature
$1.0$, top-$p{=}1.0$, top-$k{=}-1$, DAPO-Math-17K training data \citep{yu2025dapo}, on four GH200 GPUs. The Qwen3-4B experiments use the
same learning rate and temperature with max response length $4096$,
on four GH200 GPUs. In both cases, the rollout budget is $512$ per
online training step. PPO mini-batch size is $256$ for single-rollout algorithms and $32$ for 8-rollout algorithms, respectively, so that each training step performs $2$ PPO inner updates. We did not implement KL regularization, following the recent consensus \citep{yu2025dapo}.

\section{Extensions to Continuous Rewards}
\label{app:continuous-rewards}
BASIS requires to estimate the reward variance for implementing batchwise refinement. 
For binary rewards, we can use the plug-in estimate
$\widehat\sigma_i^2=\widehat V_{i,t}(1-\widehat V_{i,t})$. For a continuous scalar reward, however, the reward variance cannot be recovered from its value function alone. We instead estimate the conditional second moment
$M_{2,j,t}:=\mathbb E[r_j^2\mid x_j]$ to derive the reward variance
\begin{equation*}
\sigma_j^2=M_{2,j,t}-V_{j,t}^2.
\end{equation*}

This additional quantity can be obtained from the same offline reference-policy rollouts used to estimate the value. Indeed, applying the exponential-tilting identity in Proposition~\ref{prop:soft-value} to the first and second reward moments gives
\begin{align*}
V_\beta^*(x)
&=
\frac{\mathbb E_{\pi_{\mathrm{ref}}}
\!\left[r(x,y)e^{r(x,y)/\beta}\mid x\right]}
{\mathbb E_{\pi_{\mathrm{ref}}}
\!\left[e^{r(x,y)/\beta}\mid x\right]},\\
M_{2,\beta}^*(x)
&=
\frac{\mathbb E_{\pi_{\mathrm{ref}}}
\!\left[r(x,y)^2e^{r(x,y)/\beta}\mid x\right]}
{\mathbb E_{\pi_{\mathrm{ref}}}
\!\left[e^{r(x,y)/\beta}\mid x\right]}.
\end{align*}

Suppose the offline cache contains $M$ reference-policy rewards $r_{j1},\ldots,r_{jM}$ for prompt $x_j$, and let
$q_{jm}(\beta)=\exp(r_{jm}/\beta)$, we compute
\begin{align*}
\widehat V_\beta(x_j)
&=
\frac{\sum_{m=1}^{M}q_{jm}(\beta)r_{jm}}
{\sum_{m=1}^{M}q_{jm}(\beta)},\\
\widehat M_{2,\beta}(x_j)
&=
\frac{\sum_{m=1}^{M}q_{jm}(\beta)r_{jm}^2}
{\sum_{m=1}^{M}q_{jm}(\beta)},\\
\widehat\sigma_\beta^2(x_j)
&=
\max\!\left\{
\widehat M_{2,\beta}(x_j)-\widehat V_\beta(x_j)^2,
\delta
\right\},
\end{align*}
where $\delta>0$ is a small constant. 

At each online step $t$, we set
$\widehat V_{j,t,\beta}=\widehat V_\beta(x_j)$ and
$\widehat\sigma_{j,t,\beta}^2=\widehat\sigma_\beta^2(x_j)$ and substitute these estimates into the formula from Proposition~\ref{proBLUB}, yielding: 
\begin{align}
\widehat w_{ij,t,\beta}
&=
\frac{
\widehat V_{i,t,\beta}\widehat V_{j,t,\beta}/
\widehat\sigma_{j,t,\beta}^2
}{
\displaystyle
\sum_{k\ne i}
\widehat V_{k,t,\beta}^2/
\widehat\sigma_{k,t,\beta}^2
},
\notag\\
\widetilde V_{i,t,\beta}
&=
\sum_{j\ne i}\widehat w_{ij,t,\beta}r_j.
\label{eq:continuous-basis}
\end{align}
If the denominator in Eq.~\eqref{eq:continuous-basis} is zero, we use the same zero baseline as in the binary reward setting. Finally, $\beta_t$ is selected from the current online batch by Eq.~\eqref{eq:online-beta-selection}.


We further implement this extension using the HH dataset, which contains preferred and rejected responses. We partition the data into three disjoint subsets, containing 5,000 preference pairs for training a reward model, 2,000 preferred responses for supervised fine-tuning, and 2,048 prompts for RL policy optimization. For each subset, we reserve 5\% of the data for validation.  

We initialize the policy and reward model from Qwen2.5-1.5B-Instruct, fine-tune the policy on preferred responses for 200 steps, and train a Bradley--Terry reward model on paired preferences for 1,000 steps. For RL policy optimization, all methods use maximum prompt and response lengths of 512 tokens, sampling temperature 1.0, learning rate $5\times10^{-7}$. BASIS constructs its offline prompt values from 32 responses per prompt sampled from the supervised policy and scored by the reward model. We standardize the reward-model scores using the mean and standard deviation computed from these offline scores, subtracting the offline mean and dividing by the standard deviation of these scores. Figure~\ref{fig:hh-rlhf-continuous} reports the resulting standardized rewards. We observe that BASIS achieves substantially higher rewards than both REINFORCE and REINFORCE++ at later stages of training.

\begin{figure}[t]
    \centering
    \includegraphics[width=\columnwidth]{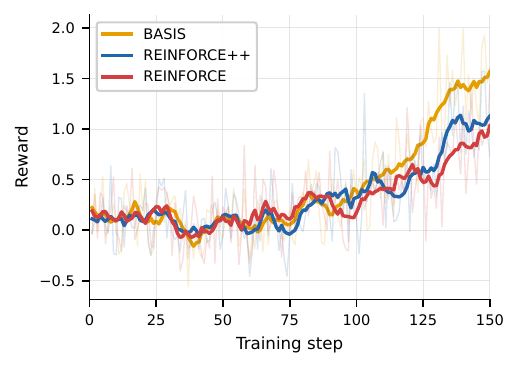}
    \caption{Training rewards on the HH dataset for BASIS and the single-rollout REINFORCE and REINFORCE++ baselines.}
    \label{fig:hh-rlhf-continuous}
\end{figure}

\section{Proofs}
\label{app:proof}

\subsection{Proof of Proposition~\ref{proBLUB}}
\begin{proof}
Fix a training step \(t\), a target prompt \(i\), and condition on the prompt
batch \(\mathcal X_B:=\{x_k\}_{k=1}^B\). Under the independent rollout
sampling used by BASIS, the rewards from different prompts
are conditionally uncorrelated given the prompt batch. Hence, for any
leave-one-out linear baseline \(b_i=\sum_{j\ne i}w_{ij}r_j\),
\[
\mathbb E[b_i\mid \mathcal X_B]
=
\sum_{j\ne i}w_{ij}V_{j,t}.
\]
The unbiasedness constraint is therefore
\[
\sum_{j\ne i}w_{ij}V_{j,t}=V_{i,t}.
\]
On this constrained set, the conditional MSE equals the conditional variance:
\[
\begin{aligned}
\mathbb E[(V_{i,t}-b_i)^2\mid \mathcal X_B]
&=
\operatorname{Var}(b_i\mid \mathcal X_B)\\
&=
\sum_{j\ne i}w_{ij}^2\sigma_j^2 .
\end{aligned}
\]
Thus the optimal weights solve the quadratic program
\[
\begin{aligned}
\min_{(w_{ij})_{j\ne i}}\quad
&\sum_{j\ne i}w_{ij}^2\sigma_j^2\\
\text{s.t.}\quad
&\sum_{j\ne i}w_{ij}V_{j,t}=V_{i,t} .
\end{aligned}
\]
In the non-degenerate case \(\sum_{k\ne i}V_{k,t}^2/\sigma_k^2>0\), the
Lagrangian
\[
\begin{aligned}
\mathcal L(w_i,\lambda)
=
\sum_{j\ne i}w_{ij}^2\sigma_j^2
+\lambda\Bigl(\sum_{j\ne i}w_{ij}V_{j,t}-V_{i,t}\Bigr)
\end{aligned}
\]
has first-order condition
\[
2w_{ij}\sigma_j^2+\lambda V_{j,t}=0,
\qquad j\ne i.
\]
Therefore \(w_{ij}=-\lambda V_{j,t}/(2\sigma_j^2)\). Substituting this expression
into the unbiasedness constraint gives
\[
-\frac{\lambda}{2}\sum_{k\ne i}\frac{V_{k,t}^2}{\sigma_k^2}=V_{i,t},
\qquad
\lambda
=
-\frac{2V_{i,t}}{\sum_{k\ne i}V_{k,t}^2/\sigma_k^2}.
\]
Hence
\[
w_{ij}
=
\frac{V_{i,t}V_{j,t}/\sigma_j^2}{\sum_{k\ne i}V_{k,t}^2/\sigma_k^2},
\qquad j\ne i,
\]
which is Eq.~\eqref{eqn:wij}. Since the objective is strictly convex whenever
\(\sigma_j^2>0\), these weights are the unique minimizer.
\end{proof}

\subsection{Proof of Proposition~\ref{prop:soft-value}}
\begin{proof}
Fix a prompt \(x\) and \(\beta>0\). We first derive the optimizer of the
KL-regularized objective using the same variational calculation as in
DPO-style derivations \citep{rafailov2023dpo}. Let
\[
p_y=\pi(y\mid x),\qquad
p_y^{\mathrm{ref}}=\pi_{\mathrm{ref}}(y\mid x),
\]
where the sums below range over responses in the support of
\(\pi_{\mathrm{ref}}(\bullet\mid x)\). For this fixed prompt, the objective is
\[
\begin{aligned}
J_x(p)
&=
\sum_y p_y r(x,y)
-
\beta\sum_y p_y
\log\frac{p_y}{p_y^{\mathrm{ref}}}.
\end{aligned}
\]
The maximization is over distributions \(p\) satisfying \(\sum_y p_y=1\). Its
Lagrangian is
\[
\begin{aligned}
\mathcal L(p,\lambda)
&=
\sum_y p_y r(x,y)
-
\beta\sum_y p_y
\log\frac{p_y}{p_y^{\mathrm{ref}}}\\
&\quad
+\lambda\Bigl(\sum_y p_y-1\Bigr).
\end{aligned}
\]
For any response \(y\) with \(p_y^{\mathrm{ref}}>0\), the first-order condition
is
\[
\begin{aligned}
0
&=
\frac{\partial\mathcal L}{\partial p_y}\\
&=
r(x,y)
-
\beta\!\left(
\log\frac{p_y}{p_y^{\mathrm{ref}}}+1
\right)
+
\lambda .
\end{aligned}
\]
Solving this equation gives
\[
p_y
=
C\,p_y^{\mathrm{ref}}
\exp\!\left(r(x,y)/\beta\right),
\]
where \(C=\exp(\lambda/\beta-1)\) does not depend on \(y\). Enforcing
\(\sum_y p_y=1\) yields
\[
\begin{aligned}
C^{-1}
&=
\sum_y p_y^{\mathrm{ref}}
\exp\!\left(r(x,y)/\beta\right)\\
&=
\mathbb E_{y\sim\pi_{\mathrm{ref}}(\bullet\mid x)}
\!\left[
\exp\!\left(r(x,y)/\beta\right)
\right]
=:Z_\beta(x).
\end{aligned}
\]
Therefore the unique maximizer is
\[
\pi_\beta^*(y\mid x)
=
\frac{
\pi_{\mathrm{ref}}(y\mid x)
\exp\!\left(r(x,y)/\beta\right)
}{
Z_\beta(x)
}.
\]
The objective is concave in \(p\) because it is a linear reward term plus
\(-\beta\) times a convex KL term, so the stationary point is globally optimal.

Now substitute this optimizer into the reward value:
\[
\begin{aligned}
V_\beta^*(x)
&=
\sum_y \pi_\beta^*(y\mid x) r(x,y)\\
&=
\frac{N_\beta(x)}{Z_\beta(x)},
\end{aligned}
\]
where
\[
\begin{aligned}
N_\beta(x)
&:=
\sum_y p_y^{\mathrm{ref}} r(x,y)
\exp\!\left(r(x,y)/\beta\right).
\end{aligned}
\]
Writing the sums in \(N_\beta(x)\) and \(Z_\beta(x)\) as expectations under
\(\pi_{\mathrm{ref}}(\bullet\mid x)\) gives Eq.~\eqref{eqn:Vbetaxstar}.
\end{proof}

\subsection{Proof of Proposition~\ref{prop:vop}}
\begin{proof}
Fix a training step \(t\), a target prompt \(i\), and condition on the
prompt batch \(\mathcal X_B\). For a leave-one-out linear baseline
\(b_i=\sum_{j\ne i}w_{ij}r_j\), the conditional MSE decomposes as
\[
\begin{aligned}
&\mathbb E[(V_{i,t}-b_i)^2\mid \mathcal X_B]
\\
&\quad =
\Bigl(V_{i,t}-\sum_{j\ne i}w_{ij}V_{j,t}\Bigr)^2
+\sum_{j\ne i}w_{ij}^2\sigma_j^2 .
\end{aligned}
\]
This is the same quadratic objective as in the proof of
Proposition~\ref{proBLUB}, but without the linear unbiasedness
constraint. Its first-order condition is
\[
\begin{aligned}
&-2V_{j,t}\Bigl(V_{i,t}-\sum_{k\ne i}w_{ik}V_{k,t}\Bigr)
+2w_{ij}\sigma_j^2=0,
\\
&\qquad j\ne i .
\end{aligned}
\]
Let \(c=V_{i,t}-\sum_{k\ne i}w_{ik}V_{k,t}\). Then
\[
w_{ij}=\frac{cV_{j,t}}{\sigma_j^2}.
\]
Substituting this expression into the definition of \(c\) gives
\[
c
=
V_{i,t}
-c\sum_{k\ne i}\frac{V_{k,t}^2}{\sigma_k^2},
\qquad
c
=
\frac{V_{i,t}}
{1+\sum_{k\ne i}V_{k,t}^2/\sigma_k^2}.
\]
Therefore
\[
w_{ij}
=
\frac{V_{i,t}V_{j,t}/\sigma_j^2}
{1+\sum_{k\ne i}V_{k,t}^2/\sigma_k^2},
\qquad j\ne i,
\]
which proves the claim. The objective is strictly convex whenever
\(\sigma_j^2>0\), so the minimizer is unique.
\end{proof}

\end{document}